\documentclass{article}

\PassOptionsToPackage{round}{natbib}

\usepackage[main, final]{neurips_2026}

\usepackage[utf8]{inputenc} 
\usepackage[T1]{fontenc}    
\usepackage[hidelinks]{hyperref}       
\usepackage{url}            
\usepackage{booktabs}       
\usepackage{amsfonts}       
\usepackage{nicefrac}       
\usepackage{microtype}      
\usepackage{xcolor}         

\usepackage{amsmath}
\usepackage{amssymb}
\usepackage{amsthm}
\usepackage{mathtools}
\usepackage{calc}
\usepackage{dsfont}
\usepackage{multirow}
\usepackage{graphicx}
\usepackage[figuresleft]{rotating}
\usepackage{wrapfig}
\usepackage{subcaption}
\usepackage{hyperref}
\usepackage{enumerate}
\usepackage{algorithmic}
\usepackage{algorithm}
\usepackage{threeparttable}
\usepackage{tabularx}
\usepackage[capitalize,noabbrev]{cleveref}

\theoremstyle{plain}
\newtheorem{theorem}{Theorem}[section]
\newtheorem{proposition}[theorem]{Proposition}

\newtheorem{corollary}[theorem]{Corollary}
\theoremstyle{definition}

\theoremstyle{remark}
\newtheorem{remark}[theorem]{Remark}

\newtheorem{taggedpropositionx}{Proposition}

\newtheorem{taggedtheoremx}{Theorem}

\newtheorem{taggedcorollaryx}{Corollary}

\def\calA{\mathcal A}

\def\calF{\mathcal F}

\def\calI{\mathcal I}

\def\calL{\mathcal L}
\def\calM{\mathcal M}
\def\calN{\mathcal N}

\def\calR{\mathcal R}

\def\calW{\mathcal W}

\def\bE{\mathbb E}

\def\bM{\mathbb M}

\def\bP{\mathbb P}
\def\bQ{\mathbb Q}
\def\bR{\mathbb R}

\newcommand\restr[2]{{
  \left.\kern-\nulldelimiterspace 
  #1 
  \vphantom{\big|} 
  \right|_{#2} 
  }}

\newcommand{\KL}{\operatorname{D_{KL}}}

\DeclareMathOperator*{\argmin}{arg\,min}

\newcommand*{\T}{^{\mkern-1.5mu\mathsf{T}}}  
\newcommand{\given}{\;\middle|\;}

\renewcommand{\tabularxcolumn}[1]{m{#1}}

\title{Reflected Schrödinger Bridge Matching}

\author{%
  Marcus Häggbom\thanks{Equal contribution.}\\
  SEB Group \textit{and} \\
  Department of Mathematics\\
  KTH Royal Institute of Technology\\
  Stockholm, Sweden \\
  \texttt{haggbo@kth.se} \\
  \And
  Viktor Nilsson\footnotemark[1] \\
  Department of Mathematics\\
  KTH Royal Institute of Technology\\
  Stockholm, Sweden \\
  \texttt{vikn@kth.se} \\
  \And
  Pierre Nyquist \\
  Department of Mathematical Sciences\\
  Chalmers University of Technology \textit{and}\\
  University of Gothenburg \textit{and}\\
  Department of Mathematics\\
  KTH Royal Institute of Technology\\ 
  \texttt{pnyquist@chalmers.se} \\
  \And
  Joakim Andén \\
  Department of Mathematics\\
  KTH Royal Institute of Technology\\
  Stockholm, Sweden \\
  \texttt{janden@kth.se} \\
}

\begin{document}

\maketitle

\begin{abstract}
    Recent advances in generative modeling have enabled the efficient computation of Schrödinger bridges (SB) in high-dimensional settings by leveraging partially simulation-free training methods inspired by flow matching.
    However, these have not covered SBs with reflecting dynamics, a useful model choice with built-in guarantees that generated samples stay in the data domain.
    Existing alternatives for reflected SBs instead rely on more complex training based on forward--backward SDE theory, requiring expensive higher-order derivatives and sampling entire paths during training.
    In this article, we introduce a partially simulation-free framework that allows reflected SBs to be trained similarly to flow matching, using a new sampling method and regression target.
    We demonstrate our results by coupling pairs of well-known high-dimensional image datasets.
    Using reflected dynamics incurs negligible additional wall-clock time during both training and inference while maintaining generative performance.
\end{abstract}
\section{Introduction}
Recent state-of-the-art progress in high-dimensional image generation has focused on iterative refinement methods such as diffusion models \citep{sohl2015deep, song2019generative, ho2020denoising, song2020score, dhariwal2021diffusion, karras2022elucidating} and closely related methods such as flow matching \citep{lipman2023flow, albergo2023stochastic}.
In \citet{liu2022flow} and \citet{liu2022rectified}, flow matching is brought closer to optimal transport theory by adding an iterative refitting step to the framework, known as \emph{rectified flows}.
This straightens out paths, but still offers no optimal transport (OT) guarantees.
To obtain these, a more suitable approach is given by Schrödinger bridges (SB), which are a probabilistic framework for coupling two arbitrary distributions $\pi_0, \pi_1$ that generalizes entropic optimal transport \citep{cuturi2013}. 

In theory, SBs can be seen as regularizations of an OT plan, where the OT plan is obtained in the low diffusion limit.
This limit is well-established via weak convergence and large deviations theory \citep{mikami2004monge, leonard2012schrodinger, bernton2022entropic, nilsson2025large}.
In practice, SBs can be learned as \emph{diffusion SBs} (DSB) \citep{bortoli2021diffusion} using iterative proportional fitting (IPF).
A recent, more efficient alternative is iterative Markovian fitting (IMF) \citep{shi2023diffusion, peluchetti2023diffusion}, where models are trained in a manner very similar to rectified flows.
Another, more flexible, version of this method is known as $\alpha$-IMF \citep{debortoli2024schrodinger}.

Most works on SBs, as well as diffusion models, have focused on unconstrained forward/reference processes.
Recently, these models have been adapted to natively handle generation on a \emph{constrained} domain $D \subseteq \bR^d$, such as the unit cube, using \emph{reflection}.
For example, \citet{lou2023reflected} introduce \emph{reflected forward/reverse processes} to the diffusion model framework, constraining them to the unit hypercube, and showcase practical and theoretical gains for classifier-free guidance.
In contrast, \citet{deng2024reflected} introduce a method of learning DSBs under \emph{reflected reference processes} based on forward--backward SDE theory introduced in \citet{chen2021likelihood}.
However, this method is very inefficient as it (1) requires full simulation for every training sample and (2) uses an expensive loss function which includes a divergence term.

We refer to methods such as rectified flows and ($\alpha$-)IMF as \emph{partially simulation-free}, since simulation of (SDE) paths is only needed for (iteratively) building the training dataset, whereas the training samples fed into the loss are generated cheaply, as in e.g. score matching.
Our goal in this paper is to take the efficient training method of ($\alpha$-)IMF and bring it to reflected Schrödinger bridges, specifically where the reference process is a reflected Brownian motion (RBM) on the unit hypercube.

\paragraph{Notation}
We work on the canonical path space of continuous functions on \([0, 1]\) taking values in \(\bR^d\), \(\Omega = C([0, 1], \bR^d)\).
Let \(X\) denote the canonical process \(X_t(\omega)\coloneqq\omega_t \in \bR^d\).
For a path measure \(\bP\), denote the projection \(\bP_t \coloneqq (X_t)_\# \bP\), and similarly \(\bP_{s, t} \coloneqq (X_s, X_t)_\# \bP\).
Furthermore, \(\bP_{\mid s}(\,\cdot\mid x_s)\) is the conditional path law \(\bP(\, \cdot \mid X_s = x_s)\), and \(\bP_{t \mid s}\) is the conditional distribution of \(X_t\) given \(X_s\).
We will assume \(\bP_{t \mid s}\) has a density with respect to the Lebesgue measure, and abuse notation by using the same notation for the density.
The conditional distribution \(\bP_{\mid 0, 1}\) is called the \emph{bridge} of \(\bP\).
For a distribution \(\Pi_{0, 1}\) on \(\bR^d \times\bR^d\), \(\bP_{\mid 0, 1} \Pi_{0, 1} \coloneqq \int \bP_{\mid 0, 1}(\,\cdot\mid x_0, x_1)d\Pi_{0, 1}(x_0, x_1)\) is a \emph{bridge mixture}.
\(\bP^r\) denotes a path measure over \(C([0, 1], D)\) corresponding to a reflected process in the domain \(D \subseteq \bR^d\).
We call the collection of all possible bridge mixtures the \emph{reciprocal class} \(\calR(\bP)\) of \(\bP\), and denote \(\calM\) (\(\calM^r\)) as the set of Markov path measures of (in \(D\) reflected) diffusion processes.
When the subscript is missing, \(\nabla\) is short for \(\nabla_{x_t}\). 

\begin{figure}
    \centering
    \includegraphics[width=1\linewidth]{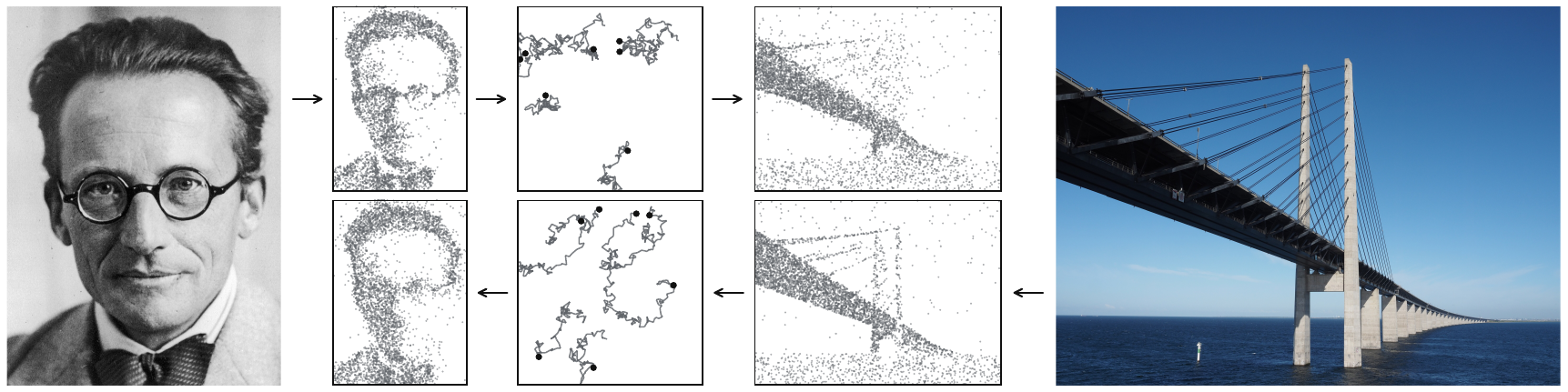}
    \caption{A reflected Schrödinger bridge between two 2D distributions resembling Erwin Schrödinger (Public Domain) and the Öresund bridge \citep[CC BY-SA 4.0]{hajotthu_oresund_bridge_2015}. The densities are derived from the pixel values of the images. Synthetic samples from one distribution are generated by evolving a reflected SDE started at a point sampled from the other.}
    \label{fig:SB-SB}
\end{figure}

\section{Related work}
\subsection{Diffusion Schrödinger bridges and bridge matching}

The first method of approximating high-dimensional SBs using deep learning was introduced by \citet{bortoli2021diffusion}.
Their method builds on an adaptation of the Sinkhorn/IPF algorithm \citep{cuturi2013}, which interleaves forward and backward network training procedures, iteratively retraining the models with respect to the reverse-time score of the other model, alternating between respecting either of the two marginal constraints of the SB.

\citet{shi2023diffusion} and \citet{peluchetti2023diffusion} established that DSBs (with an SDE reference process) can in fact be learned in a manner highly similar to flow matching (especially the rectified flows framework \citep{liu2022flow}) called \emph{DSB matching} (DSBM)\footnote{\citet{peluchetti2023diffusion} calls it the \textit{Iterated Diffusion Bridge Mixture} procedure.}.
In this method, which is presented in more detail below, one trains the score network against the drift of the reference bridge processes, obtained by the Doob $h$-transform \citep{rogers2000diffusions}, where the endpoints are sampled from a coupling between the two desired marginals.
The SDE with the learned drift added to the dynamics preserves the marginal distributions (under perfect training), and thus the joint distribution, of time $0$ and $1$ samples from this process is still a valid coupling.
Thus the marginal constraints remain satisfied.
The authors showed that iterating this procedure (called iterative Markovian fitting (IMF)) with the obtained coupling will converge to the Schrödinger bridge.
The whole procedure can be done either backwards or forwards in time.

If the reference process used in DSBM is a Brownian motion, the target $h$-transform term will be 
\begin{equation}\label{eq:dSB-matching-Brownian-target}
    \frac{y-X_t}{1-t},
\end{equation}
and $X_t$ is conditionally sampled from a Brownian bridge between $x$ and $y$ at time $t \in [0,1]$, i.e. from $\calN((1-t)x + ty, \sigma^2 t (1-t))$.
If the variance $\sigma^2$ is taken to $0$, this becomes $(1-t)x + ty$, and the target \eqref{eq:dSB-matching-Brownian-target} is $y-x$ for all $t$, which is precisely the target for flow matching with straight conditional paths, used in \citet{lipman2023flow} and \citet{liu2022flow}.
In \citet{liu2022rectified}, it is shown that iterative training of a flow-matching vector field (in the fashion laid out for IMF) converges to a straight path interpolation between the two distributions. 
Despite this, these are not necessarily OT plans, while Schrödinger bridges for the appropriate reference processes are indeed true EOT plans, as was further analyzed in \citet{hertrich2026relation}.

\subsection{Constrained generative modeling}
The task of adapting diffusion models to constrained generation has been approached from various directions \citep[e.g.][]{liu2023mirror,fishman2023diffusion,fishman2023metropolis,christopher2024constrained}.
The treatment closest to our work is that of \citet{lou2023reflected}, who proposed a simulation-free method for learning a diffusion model where the whole process is constrained to the unit hypercube, i.e. $D = [0, 1]^d \subseteq \bR^d$.
The forward process is a \emph{reflected SDE} which has the form
\begin{equation}\label{eq:reflected-forward-SDE}
    dX_t = b(t, X_t) dt + \sigma(t, X_t) dW_t + n(X_t) dL_t, \quad X_0 \sim \mu.
\end{equation}
Here, $\{W_t\}_t$ is a Brownian motion independent of $X_0$, $n(x)$ is an inward pointing normal at $x$ when $x$ is on the boundary $\partial D$ and else zero, and ${L_t}$ is a \emph{local time} for the boundary which increases only while $X_t$ is on the boundary, exactly so that $X_t$ stays in the domain. 
The time-reversal $\{\overline{X}_t \coloneq X_{1-t}\}$ of \eqref{eq:reflected-forward-SDE} is given by \citep{cattiaux1988time}:
\begin{equation}\label{eq:reflected-reverse-SDE}
    d\overline{X}_t = (-b(1-t, \overline{X}_t) + \overline{g}(t, \overline{X}_t)) dt + \sigma(1-t, \overline{X}_t) d\overline{W}_t + n(\overline{X}_t) d\overline{L}_t,
\end{equation}
where $\overline{g}(t, x) \coloneq \sigma^2(1-t) \nabla \log p(1-t, x)$ (for $\sigma(t,x)=\sigma(t)\in \bR$), and $p(t, x)$ is the density of $\{X_t\}$.
Thus, one has to learn $\nabla \log p(1-t, x)$ in order to generate data with the reverse process \eqref{eq:reflected-reverse-SDE}.
The authors showed that this is possible by utilizing the Doob $h$-transform of the forward process \eqref{eq:reflected-forward-SDE}.
In numerical experiments, they showed that this setup is greatly advantageous for mitigating oversaturation effects in the generated samples, especially under classifier-free guidance \citep{ho2022classifier}.

\citet{liu2023learning} proposed a method for constrained sampling of diffusion (not Schrödinger) bridges.
By means of Doob's \(h\)-transform, a term in the drift of the (non-reflected) forward process is added, forcing the samples to a predefined support.
This is used to sample discrete distributions, with compact intervals for support being another possibility.

\citet{caluya2021reflected} were to our knowledge the first to treat the reflected Schrödinger problem, working with the control formulation of the problem.
\citet{deng2024reflected} then built upon this work, taking the problem to a high-dimensional generative modeling setting by incorporating and adjusting the costly forward-backward SDE framework SB-FBSDE of \citet{chen2021likelihood}.
In the Reflected SB-FBSDE (RSB-FBSDE) of \citet{deng2024reflected}, unlike in \citet{lou2023reflected} and \citet{shi2023diffusion}, the loss is not a regression of the analytic score or score-like function.
This entails a greater flexibility in terms of reference processes and support domains.

\section{Schrödinger bridge matching}

In this section we present the framework on which the IMF \citep{shi2023diffusion, peluchetti2023diffusion} and $\alpha$-IMF \citep{debortoli2024schrodinger} algorithms are constructed.

The goal of IMF is to approximate a Schrödinger bridge $\bP^*$ with respect to some reference measure $\bQ$ between two distributions $\pi_0$ and $\pi_1$, defined as 
\begin{equation}
    \label{eq:schrödinger-bridge-problem}
    \bP^* = \argmin_\bP \left\{\KL (\bP \mid\mid \bQ) : \bP_0 = \pi_0, \bP_1 = \pi_1 \right\}.
\end{equation}
The reference measure $\bQ$ is in general the law of an SDE
\begin{equation}
    \label{eq:def-DSBM-Q-SDE}
    dX_t = \mu(t, X_t) dt + \sigma(t) dW_t, \quad X_0 \sim \bQ_0.
\end{equation}
At its core, the algorithm consists of two steps -- \textit{Markovian projection} and \textit{reciprocal projection} -- that are alternated.
We say that a measure $\Pi$ is in the \textit{reciprocal class} \(\calR(\bQ)\) of $\bQ$ if $\Pi$ is a mixture of $\bQ$-bridges, i.e. $\Pi = \Pi_{0, 1} \bQ_{\mid 0, 1}$, and define the reciprocal projection as $\operatorname{proj}_{\calR(\bQ)}(\Pi) := \Pi_{0, 1} \bQ_{\mid 0, 1}$.
It can be shown that under certain conditions, the Schrödinger bridge with reference measure $\bQ$ is the unique path measure that is simultaneously Markovian, in the reciprocal class of $\bQ$, and satisfies the marginal constraints \citep{leonard2014survey, shi2023diffusion}.

In general, the reciprocal projection is not Markovian. Assuming some regularity, however, for a given mixture of bridges $\Pi = \Pi_{0, 1} \bQ_{\mid 0, 1}$ one can construct a Markov process whose marginal law at time \(t\) is equal to $\Pi_t$.
This is the Markovian projection step $\operatorname{proj}_\calM(\Pi)$.
The resulting Markov process is the SDE 
\begin{equation}
    \label{eq:nonrefl-markovian-proj}
    \begin{gathered}
        dX_t = \left(\mu(t, X_t) + v^*(t, X_t)\right) dt + \sigma(t) dW_t, \quad X_0 \sim \Pi_0,\\
        v^*(t, x) = \sigma^2(t) \, \bE_{\Pi_{1 \mid t}} \left[ \nabla \log \bQ_{1 \mid t} (X_1 \mid X_t) \given X_t = x \right],
    \end{gathered}
\end{equation}
where the gradient is taken with respect to the second argument $x_t$.

Just as the reciprocal projection breaks the Markov property, the Markovian projection breaks the bridge mixture.
Instead, the algorithm preserves the marginal distributions at $t=0$ and $1$.
Initializing the algorithm at $(\pi_0 \otimes \pi_1)\, \bQ_{\mid 0, 1}$, if we converge to a fixed point then it is both Markov and in the reciprocal class, and is thus the Schrödinger bridge.

DSBM is a numerical implementation of the IMF algorithm.
The main challenge is the Markovian projection step, where the drift correction $v^*$ is approximated by a neural network and trained on a regression loss similar to diffusion model score matching \citep{song2019generative}.
The training procedure is initialized in a mixture of bridges between the two datasets, used in the first iteration of the Markovian projection.
Taking the reciprocal projection then amounts to sampling pairs at $t=0, 1$ from this Markov process and then sampling bridges between the points in each pair.
These are then stored in a cache to train the subsequent Markovian projection.
Naively implementing this will result in a compounded error for $\pi_1$, but due to the symmetry of the problem one can formulate a reverse-time Markovian projection where $\pi_0$ is instead the target, creating a bidirectional training scheme.

Implementing IMF requires updating the training-data cache when the optimal \(v^*\) has been found and the Markovian projection has been approximated. 
The $\alpha$-IMF algorithm is an extension of IMF, where training data is instead updated continuously in an online training fashion.
Here, the iterates are
\begin{equation}
    \Pi^{n+1} = (1 - \alpha) \Pi^n + \alpha \operatorname{proj}_{\calR(\bQ)}(\operatorname{proj}_\calM(\Pi^n)),
\end{equation}
which like IMF can be shown to converge to the Schrödinger bridge \citep{debortoli2024schrodinger}.
In the corresponding implementation $\alpha$-DSBM, the Markovian projection optimization is not minimized fully before resuming, and the parameter $\alpha$ is implicit in the (possibly dynamic) step size of the optimizer.
The training procedure also more naturally enables the forward and backward gradient fields to be parametrized by the same network and being optimized for the forward and backward losses simultaneously.
The algorithm is explained in more detail in the following section, there in its reflected version.

\section{Reflected Schrödinger bridge matching}
\label{sec:rsbm}

\subsection{Extension of IMF}
As the name suggests, RSBM is an extension of DSBM where the reference process to the SB problem \eqref{eq:schrödinger-bridge-problem} is a process \(\bQ^r\) defined by the reflected SDE
\begin{equation}
    \label{eq:def-Qr-SDE}
    dX_t = \mu(t, X_t) dt + \sigma(t) dW_t + n(X_t) dL_t, \quad X_0 \sim \bQ^r_0,
\end{equation}
constrained to some compact domain \(D\), on which \(\pi_0\) and \(\pi_1\) are assumed to be supported.
For a formal introduction to reflected SDEs, we refer to \citet{pilipenko2014introduction}. 
In a nutshell, $n(x)$ is the inward-pointing normal at $x$ and the process \(L\), called the \emph{local time} of \(X\) \citep{bjork2015pedestriansguidelocaltime}, compensates the process pathwise at the boundary, ensuring it does not escape.
When \(X\) is in the interior of \(D\), \(L\) is constant and the process behaves like the regular, non-reflected SDE with the same drift and diffusion coefficients.

From the PDE perspective, the density of the reflected process satisfies the same Fokker--Planck equation in \(D\) as the non-reflected, but with an additional Neumann boundary constraint.
The intuition is that the flow of the density over the boundary is zero, ensuring that no mass escapes and the support remains confined to \(D\).

From a numerical point of view, the reflected SDE is estimated using Euler--Maruyama discretization where at each time step, the process is either reflected back into or projected onto the domain.
As in \citet{lou2023reflected}, we opt for reflection, with an exception for the image tasks where the final step is projected.

In the IMF procedure, when the reference process is a reflected SDE, the two types of projections are analogous to the original method.
The reciprocal projection is identical, i.e., \(\operatorname{proj}_{\calR(\bQ^r)}(\Pi) := \Pi_{0, 1} \bQ^r_{\mid 0, 1}\).
In practice, the reference process therefore needs to be on a sufficiently nice form that sampling the bridge is efficient.

The Markovian projection \(\operatorname{proj}_{\calM^r}(\Pi)\) in the reflected case is unsurprisingly defined as the reflected SDE
\begin{equation}
    \label{eq:def-sde-of-refl-markov-proj}
    \begin{gathered}
        dX_t = (\mu(t, X_t) + v^*(t, X_t)) dt + \sigma(t) dW_t + n(X_t) dL_t, \quad X_0 \sim \Pi_0, \\
        v^*(t, x) = \sigma^2 (t) \, \bE_{\Pi_{1 \mid t}} \left[ \nabla \log \bQ^r_{1 \mid t} (X_1 \mid X_t) \given X_t = x \right].
    \end{gathered}
\end{equation}
Comparing with \eqref{eq:nonrefl-markovian-proj}, here the drift is adjusted by an expectation of the score of the transition density of the reflected reference process \(\bQ^r_{1 \mid t}\).
Preservation of marginal distributions holds due to the following result which is an extension of the results of the non-reflected case \citep[cf.][Proposition~2]{shi2023diffusion}.
\begin{proposition}
    \label{prop:markov-proj-refl-main}
    Assume \(\sigma > 0\) and that the solutions to the involved RSDEs exist and are unique.
    With some regularity assumptions on \(d\Pi / d\bQ^r\) (see Appendix~\ref{sec:appendix-proofs}), letting \(\bM^* = \operatorname{proj}_{\calM^r}(\Pi)\), the following hold:
    \begin{enumerate}[(i)]
        \item \(\bM^*_t = \Pi_t \quad \text{for all } t \in [0, 1].\)
        \item \(\bM^* = \argmin_{\bM \in \calM^r} \KL (\Pi \mid\mid \bM).\)
        \item \begingroup
            \setlength{\abovedisplayskip}{0pt}
            \setlength{\abovedisplayshortskip}{0pt}
            \setlength{\belowdisplayskip}{0pt}
            \begin{equation}\begin{multlined}
                \label{eq:kl-div-for-markov-proj}
                \KL (\Pi \mid\mid \bM^*) \\
                = \frac{1}{2}\int_0^1 \frac{1}{\sigma^2(t)}\bE_{\Pi_{0, t}}\left[ \left\|\sigma^2(t) \bE_{\Pi_{1 \mid 0, t}} \left[ \nabla \log \bQ^r_{1 \mid t}(X_1 | X_t) \mid X_0, X_t \right] - v^*(t, X_t)\right\|^2\right] dt.
            \end{multlined}\end{equation}
        \endgroup
    \end{enumerate}
\end{proposition}
The full proof is in Appendix~\ref{sec:appendix-proofs}.
In summary, the major difference compared with \citet{shi2023diffusion} is the treatment of the Doob \(h\)-transform, which is used in retrieving the dynamics of the bridge process \(\Pi_{\mid 0}\) conditioned on the starting point.
Whereas \citet{shi2023diffusion} produce this dynamic from the infinitesimal generator \citep{palmowski2002technique}, we use Girsanov's theorem to circumvent the issue of identifying the domain of the generator in the constrained regime.
The remainder of the proof follows \citet{shi2023diffusion}, adapting the tools used in their proof to the reflected setting as needed.

Under the corresponding assumptions required for IMF convergence, Proposition~\ref{prop:markov-proj-refl-main} allows the convergence argument of \citet{shi2023diffusion} to be carried over to the reflected setting.
We are then able to state the implementation, which is identical to \(\alpha\)-DSBM except for the loss function and sampling.
The main idea of DSBM is to iterate the IMF, where at each Markovian projection step, \(v^*\) in \eqref{eq:kl-div-for-markov-proj} is approximated with a neural network \(v_\theta\) trained on the objective of minimizing this KL divergence.
The subsequent reciprocal projection is sampled from by first drawing a starting point \(X_0 \sim \Pi_0 = \pi_0\) from the training data, sampling \(X_1 \mid X_0\) from the approximate Markovian projection
\begin{equation}
    \label{eq:def-sde-of-approx-refl-markov-proj}
    X_1 = X_0 + \int_0^1 (\mu(t, X_t) + v_\theta(t, X_t)) dt + \int_0^1 \sigma(t) dW_t + \int_0^1 n(X_t) dL_t.
\end{equation}
with Euler--Maruyama, then sampling a bridge point \(X_t \mid X_0, X_1 \sim \bQ^r_{t \mid 0, 1}(\, \cdot \mid X_0, X_1)\).

We follow the extension \(\alpha\)-DSBM, where only one gradient step is taken at each Markovian projection step in an online fashion.
Whereas the forward and backward objectives can be alternated, we here present the case of a bidirectional model where the direction is passed together with the time \(t\) as a conditional parameter, training on the forward and backward objective simultaneously.
This gives rise to the loss
\begin{equation}
    \label{eq:rsbm-bidirectional-loss}
    \begin{gathered}
        \calL (\theta ; \Pi) = \frac{1}{2} \left( \calL _{\to}(\theta ; \Pi) + \calL _{\gets}(\theta ; \Pi) \right), \quad \text{ where} \\
        \begin{aligned}
            \calL _{\to}(\theta ; \Pi) &= \int_0^1 \frac{1}{\sigma^2(t)}\bE_{\Pi_{t, 1}}\left[\| \sigma^2(t)  \nabla \log \bQ^r_{1 \mid t}(X_1 \mid X_t) - v_\theta(t, X_t; \to) \|^2\right] dt, \\
            \calL _{\gets}(\theta ; \Pi) &= \int_0^1 \frac{1}{\sigma^2(t)}\bE_{\Pi_{0, t}}\left[\| \sigma^2(t)  \nabla \log \bQ^r_{t \mid 0}(X_t \mid X_0) - v_\theta(1-t, X_t; \gets) \|^2\right] dt.
        \end{aligned}
    \end{gathered}
\end{equation}

The \(\alpha\)-RSBM algorithm is described in \cref{alg:rsbm} \citep[cf.][Algorithm~1]{debortoli2024schrodinger}.
The RSBM algorithm, corresponding to \(\alpha = 1\), can be analogously extended from DSBM \citep[Algorithm~1]{shi2023diffusion}.
In practice, an exponential moving average of the weights is tracked and used in the Euler--Maruyama sampling step in \cref{alg:rsbm} to improve sample quality \citep{song2020improved}.

\begin{algorithm}
    \caption{\(\alpha\)-Reflected Schrödinger Bridge Matching}\label{alg:rsbm}
    \begin{algorithmic}
        \REQUIRE distributions \(\pi_0\), \(\pi_1\), reference process \(\bQ^r\);
            training parameters \(N_{\text{pretrain}}\), \(N_{\text{finetune}}\), half batch size \(b\); pretrain optimizer, finetune optimizer with learning rate \(\alpha\)
        \FOR{\(n = 1, \dots, N_{\text{pretrain}}\)}
            \STATE Sample \(2b\) endpoint pairs \((X_0, X_1) \sim \pi_0 \otimes \pi_1\)
            \STATE Sample \(2b\) time steps \(t \sim \operatorname{Unif}(0, 1)\)
            \STATE Sample \(2b\) bridge points \(X_t \mid X_0, X_1 \sim \bQ^r_{t \mid 0, 1}(\, \cdot \mid X_0, X_1)\)
            \STATE Gradient step on empirical version of loss \eqref{eq:rsbm-bidirectional-loss}, each half-batch in one direction
        \ENDFOR
        \FOR{\(n = 1, \dots, N_{\text{finetune}}\)}
            \STATE Sample \(b\) starting points \(X_0 \sim \pi_0, X_1 \sim \pi_1 \) each
            \STATE Sample \(b\) endpoints \(\hat{X}_1 \mid X_0\) by \eqref{eq:def-sde-of-approx-refl-markov-proj}
            \STATE Sample \(b\) endpoints \(\hat{X}_0 \mid X_1\) by reverse-time version of \eqref{eq:def-sde-of-approx-refl-markov-proj}
            \STATE Sample \(2b\) time steps \(t \sim \operatorname{Unif}(0, 1)\)
            \STATE Sample \(b\) bridge points \(X_t \mid X_0, \hat{X}_1 \sim \bQ^r_{t \mid 0, 1}(\, \cdot \mid X_0, \hat{X}_1)\), compute empirical \(\calL _{\gets}\)
            \STATE Sample \(b\) bridge points \(X_t \mid \hat{X}_0, X_1 \sim \bQ^r_{t \mid 0, 1}(\, \cdot \mid \hat{X}_0, X_1)\), compute empirical \(\calL _{\to}\)
            \STATE Gradient step on total empirical loss \eqref{eq:rsbm-bidirectional-loss}
        \ENDFOR
    \end{algorithmic}
\end{algorithm}

In its general form, the algorithm contains three elements that each poses a computational challenge:
First, the algorithm requires sampling an endpoint \(\hat{X}_1\) of the Markovian projection \eqref{eq:def-sde-of-approx-refl-markov-proj} (and \(\hat{X}_0\), respectively) at each finetune step.
Second, the same goes for sampling a bridge point from the reciprocal distribution, i.e. given the starting point \(X_0\) and the endpoint \(\hat{X}_1\) sampling from \(\bQ^r_{t \mid 0, 1}\).
Third, the loss contains the score function \(\nabla \log \bQ^r_{1 \mid t}(x_1 \mid x_t)\) (and the reverse-time counterpart), which is in general not available in closed form.
In the non-reflected case, the last two points are made tractable by letting the reference process be a scaled Brownian motion \(\sigma W\).
Accordingly, we deal with these in the reflected case by restricting \(\bQ^r\) to reflected scaled Brownian motions in the unit cube.

\subsection{The score function}
When \(\bQ\) is a scaled Brownian motion with no reflection, i.e. \(\mu(t, x) = 0\) and \(\sigma(t) = \sigma\) in \eqref{eq:def-DSBM-Q-SDE}, the transition density is a Gaussian and the score function is simply
\[\nabla_{x_s} \log \bQ_{t \mid s}(x_t \mid x_s) = \frac{x_t - x_s}{\sigma^2 \, (t - s)}.\]
For a reflected Brownian motion on the unit cube, we can approximate the score cheaply.
Let $R_D: D \times \bR^d \to D$ be the reflection operator for the domain $D$, which computes the endpoint of the reflection of the line from $x$ to $y$.
For $D = [0, 1]$, $R_{[0,1]}$ only depends on $y$, and is given by
\begin{equation}\label{eq:G-reflection-operator}
    R_{[0,1]}(y) = \begin{cases}
        \phantom{(-} y  & \operatorname{mod} 2, \ \text{if this is in $[0, 1]$,} \\
                 (-  y) & \operatorname{mod} 2, \ \text{otherwise}.
    \end{cases}
\end{equation}
This also holds for $D = [0,1]^d$ for any dimension $d$, and we denote this operator simply as $R$.
As in \citet{lou2023reflected}, we have for the unit cube that 
\begin{equation}
    \label{eq:reflected-transition-density}
    \bQ^r_{t \mid s}(x_t \mid x_s) = \sum_{y_t \in \calI(x_t)} \bQ_{t \mid s}(y_t \mid x_s), \quad t > s,
\end{equation}
where $\mathcal{I}(x)$ is the set of points $y$ in $\bR^d$ such that $R(y) = x$ (the inverse image of $R$).
In practice, we truncate the infinite set $\mathcal{I}(x)$ at a suitable distance.
The simplest way to do this is to pick a big number of terms and evaluate the sum on a grid of pairs \((x_0, x_1)\), and make sure that the omitted terms at the selected lower truncation level carry sufficiently low mass.  
\citet{lou2023reflected} propose a hybrid approach to the reflected transition density, using the sum-of-Gaussians above for small \(\sigma\) and for larger \(\sigma\) an expansion in Laplacian eigenfunctions.
We argue that sum-of-Gaussians is sufficient in our case, since we do not rely on variance exploding but use small to moderate values for \(\sigma\).

Note that this truncated set grows exponentially with dimension $d$.
However, when the reference process has independent components, we can instead use the factorization
\begin{equation}
    \bQ^r_{t \mid s}(x_t \mid x_s) 
    = \prod_{i=1}^d \bQ^{r, i}_{t \mid s}(x_{t}^{i} \mid x_s^i) 
    = \prod_{i=1}^d \sum_{y_t^i \in \mathcal{I}(x_t^i)} \bQ^{i}_{t \mid s}(y_{t}^{i} \mid x_s^i).
\end{equation}
Thus, we can compute the Doob $h$-transform component-wise as 
\begin{equation}
    (\nabla_{x_s} \log \bQ^r_{t \mid s}(x_t \mid x_s))_i 
    = \frac{d}{dx_s^i} \log \bQ^{r, i}_{t \mid s}(x_{t}^{i} \mid x_s^i)
    = \frac{\sum_{y_t^i \in \mathcal{I}(x_t^i)} \frac{d}{dx_s^i} \bQ^{i}_{t \mid s}(y_{t}^{i} \mid x_s^i)}{\bQ^{r, i}_{t \mid s}(x_{t}^{i} \mid x_s^i)}.
\end{equation}

\subsection{Efficient sampling of the reflected bridge}
Bridge sampling is in general not trivial either.
It is possible to use Doob's \(h\)-transform which yields an SDE with correct marginal distribution.
In some cases, as with a Brownian bridge, this SDE can be solved analytically and an efficient sampling strategy can exist.
Otherwise, one is left with sampling the SDE using the Euler--Maruyama method.
Here, we describe an efficient sampler for the reflected Brownian bridge and prove its validity. 

\begin{algorithm}
\caption{Bridge midpoint sampling for RBM}
\label{alg:midpoint-sample-RBM}
\begin{algorithmic}[1]
\REQUIRE $x_0, x_1 \in D, t \in [0, 1]$
\ENSURE $X_t \sim \bQ^r_{t\mid 0, 1}(\, \cdot \mid X_0 = x_0, X_1 = x_1$)
\STATE Draw $Y_1$ from $\calI(x_1)$ with probabilities proportional to $\bQ_{1 \mid 0}(y_1 \mid x_0)$
\STATE Draw $Y_t \mid Y_1 \sim \calN(tY_1 + (1-t)x_0, t(1-t)\sigma^2 I_d)$
\STATE \textbf{return} $R(Y_t)$
\end{algorithmic}
\end{algorithm}

\begin{proposition}
    \label{prop:bridge-sampling}
    Algorithm~\ref{alg:midpoint-sample-RBM} correctly samples the RBM bridge midpoints.
\end{proposition}

\begin{proof}
    Let \(W\) be a Brownian motion.
    Since $\{R W_t\}_t$ has the same distribution as the reflected Brownian motion $X$, the Markov chains $X_0 \to X_t \to X_1$ and $R W_0 \to R W_t \to R W_1$ have the same joint distribution, and we may work with the latter.
    Working conditionally on $W_0$, we further have the Markov chain $R W_1 \to W_1 \to W_t \to R W_t$.
    To sample $W_1 \mid R W_1 = x_1$, use Bayes' rule (also conditioning on $W_0$):
    \begin{equation}
    \begin{split}
        p_{W_1 | R W_1, W_0}(y_1 | x_1, x_0) &\propto p_{R W_1 | W_1, W_0}(x_1 | y_1, x_0) p_{W_1 | W_0}(y_1 | x_0) = \delta_{R y_1}(x_1) \bQ_{1 \mid 0}(y_1 \mid x_0),
    \end{split}
    \end{equation}
    where $\delta$ is the Dirac delta.
    With respect to the counting measure on the set $\calI(x_1) = R^{-1}(x_1)$ instead, this is then proportional to \(\bQ_{1 \mid 0}(y_1 \mid x_0)\), justifying step 1 of Algorithm~\ref{alg:midpoint-sample-RBM}.
    Step 2 is simply sampling a standard Brownian bridge to get $W_t \mid W_0 = x_0, W_1 = y_1$.
    Step 3 only requires applying $R$ to this sample.
\end{proof}

\subsection{Limitations and comparison with RSB-FBSDE}
\label{subsec:limitations}
As presented in the two previous subsections, we impose two limitations to make RSBM tractable:
The reference process is a reflected Brownian motion, i.e. the drift \(\mu\) in \eqref{eq:def-Qr-SDE} must be zero, and the domain D needs to be the unit cube.
We stress, however, that the main theoretical result holds for more general reference processes.
For a chosen reference, one needs to solve the problems of efficient bridge point sampling and score estimation.
As for the domain restriction, it can be loosened if the support can mapped with a diffeomorphism to the unit cube.
However, \citet{deng2024reflected} show this might lead to a blow-up of the transformed true distribution in certain regions, making estimation harder.
Neither of these restrictions are needed in RSB-FBSDE \citep{deng2024reflected}, which allows more general processes and domains.
This comes at a cost of having the loss depend on the divergence of the score network, requiring two network propagations for each gradient step.
The method requires SDE sampling during training, like RSBM, but in addition also relies on caching the entire path rather than just start and endpoints.

\citet{deng2024reflected} also bring up some limitations of their method in high dimension.
They conjecture that simpler domains will still be required due to the curse of dimensionality.
The high-dimensional experiments on image data are also said to be reliant on a good initialization of the score network, which is achieved by pretraining on the reflected diffusion model objective of \citet{lou2023reflected}.
This effectively means, as for RSBM, the variance-exploding SDE is selected as reference process.
For a variance-preserving reference process, \citet{deng2024reflected} suggest that a non-reflected variance-preserving diffusion model would be a good initialization candidate.
In both cases, however, one of the marginal distributions is a noise prior (uniform or truncated Gaussian).
Further, the final signal-to-noise ratios of the reference processes are so low that the endpoints are essentially independent, and the SBs have little difference from standard diffusion models.
No discussion is provided on what form of pretraining might allow for unpaired image-to-image sampling.

\section{Experiments}
\label{sec:experiments}

We evaluate the proposed \(\alpha\)-RSBM model by comparing it to RSB-FBSDE on toy datasets in 2D, as well as with the non-reflected \(\alpha\)-DSBM on two image translation problems, where support is naturally constrained to pixel values in \([0, 1]\) for grayscale or \([0, 1]^3\) for RGB images.
As discussed in previous section, RSB-FBSDE is not compared on the image-to-image tasks due to the question of initialization.
The code for reproducing the experiments is available at \url{https://github.com/viktor765/RSBM}, and trained weights at \url{https://doi.org/10.5281/zenodo.24574007}.

\subsection{Toy problems in 2D}
\begin{table}[t]
    \caption{Final 2-Wasserstein distance and control cost of the two-dimensional experiments, as the mean \(\pm\) one sample standard deviation over three seeds, using EMA weights.}
    \label{tab:toy-2d-evaluation}
    \centering
    \small
    \begin{tabular}{@{}llcccc@{}}
        \toprule
        & & \multicolumn{2}{c}{Uniform--Checkerboard} & \multicolumn{2}{c}{Schrödinger--Öresund} \\
        \cmidrule(lr){3-4} \cmidrule(lr){5-6}
        & & \(\calW_2\) \textdownarrow & Cost \textdownarrow & \(\calW_2\) \textdownarrow & Cost \textdownarrow \\
        \midrule
        \multirow{3}{*}{\(\Pi_0 \to \Pi_1\)} & \(\alpha\)-RSBM (1h) & 0.021 $\pm$ 0.005 & \phantom{0}1.97 $\pm$ 0.01 & 0.033 $\pm$ 0.008 & 15.78 $\pm$ 0.50 \\
         & RSB-FBSDE (2h) & 0.034 $\pm$ 0.005 & \phantom{0}2.08 $\pm$ 0.03 & 0.130 $\pm$ 0.003 & \phantom{0}6.83 $\pm$ 0.12 \\
         & RSB-FBSDE (6h) & 0.024 $\pm$ 0.002 & 12.35 $\pm$ 2.26 & 0.046 $\pm$ 0.007 & 16.50 $\pm$ 0.42 \\
        \cmidrule(l){1-2}
        \multirow{3}{*}{\(\Pi_0 \gets \Pi_1\)} & \(\alpha\)-RSBM (1h) & 0.020 $\pm$ 0.003 & \phantom{0}1.21 $\pm$ 0.06 & 0.032 $\pm$ 0.005 & 15.06 $\pm$ 0.74 \\
         & RSB-FBSDE (2h) & 0.023 $\pm$ 0.005 & \phantom{0}1.33 $\pm$ 0.03 & 0.127 $\pm$ 0.002 & \phantom{0}6.41 $\pm$ 0.10 \\
         & RSB-FBSDE (6h) & 0.020 $\pm$ 0.003 & 12.82 $\pm$ 1.94 & 0.059 $\pm$ 0.009 & 33.34 $\pm$ 5.91 \\
        \bottomrule
    \end{tabular}
\end{table}

We choose two problems of interest for the reflected setting in 2D by picking distributions with support in a square domain that carry non-negligible mass in the neighborhood of the boundary.
The first example is that of transporting a uniform distribution over the whole square to a uniform over the black squares of a \(4 \times 4\) checkerboard.
The second example is Schrödinger \(\leftrightarrow\) Öresund demonstrated in Figure~\ref{fig:SB-SB}, where two distribution have been created as to make sampled point clouds bare resemblance of photographs of Erwin Schrödinger and the Öresund bridge. 

The control formulation of the SB problem \citep[Eq.~(3)]{deng2024reflected} states that the Schrödinger bridge is the controlled process \(dX^u_t = (\mu(t, X^u_t) + v(t, X^u_t)) dt + \sigma(t) dW_t + n(X^u_t) dL_t\) that satisfies the marginal constraints and has the lowest \textit{control cost}
\(
\frac{1}{2} \bE \left[ \int_0^1 \frac{1}{\sigma^2(t)}\|v(t, X^u_t)\|^2 dt\right].
\)
This suggests using control cost paired with 2-Wasserstein distance for each marginal fit as evaluation metrics.
As described in more detail in Appendix~\ref{sec:appendix-training-details}, we order the outcomes by lowest control cost (summed over both directions) out of those achieving sufficiently good marginal fit in both marginals.

The results are presented in Table~\ref{tab:toy-2d-evaluation}.
We compare RSBM trained for one hour and RSB-FBSDE trained for two.
RSB-FBSDE did not show signs of convergence after the two hours, so an additional twelve hour run (with learning schedule scaled) was performed.
The weights exploded after about seven hours, but the evaluation at 6h shows significant improvement in \(\calW_2\), although at the cost of deteriorated control cost.
Neither runs were able to achieve as good of a control cost as RSBM while also having sufficient marginal fits (see also Figure~\ref{fig:long-run} in the appendix).

\subsection{Image-to-image}

We approximate an SB between MNIST \citep[Unspecified license]{mnist} and the first five letters (A--E, a--e) of EMNIST \citep[CC BY-ND 4.0]{emnist}, as well as between the classes \texttt{cat} and \texttt{wild} in the AFHQv2 dataset \citep[CC BY-NC 4.0 license]{choi2020starganv2}.
MNIST and EMNIST are both 28~\(\times\)~28 grayscale pixels.
For AFHQ we downsample to a 64~\(\times\)~64 resolution.

We follow the experimental setup of \citet{debortoli2024schrodinger}.
For quantitative evaluation, we compare the generative performance by computing the FID between the training data of \(\Pi_1\) and synthetic samples started at the test data from \(\Pi_0\), and vice versa.
As \citet{debortoli2024schrodinger} point out, one should not draw too strong conclusions about performance from the FID alone in these cases, as MNIST and EMNIST are black-and-white, and because of the relatively small test sample sizes (\(\sim 500\)) for AFHQ.
Start- and end-point similarity is quantified by the (per dimension) mean square distance (MSD) of pixel values.
For AFHQ, LPIPS \citep[learned perceptual image patch similarity, ][]{lpips} is also computed in a similar fashion to quantify this likeness with respect to closeness of AlexNet feature representations.
The experiment is explained in more detail in Appendix~\ref{sec:appendix-additional-results}.

The comparison between \(\alpha\)-DSBM and \(\alpha\)-RSBM is presented in \cref{tab:image-eval-scores}.
Overall, the reflected model preserves more likeness to the source image in all cases except in terms of LPIPS for \texttt{cat} \(\to\) \texttt{wild}.
For generative quality, the wide seed variation indicates that no strong conclusions can be drawn for either method.

Some randomly chosen generated samples are shown in \cref{fig:afhq64-mnist-random-samples}, with more examples presented in \cref{fig:mnist-more-random-samples-incl-top-oob} and \cref{fig:afhq64-more-random-samples-incl-top-oob} in Appendix~\ref{sec:appendix-additional-results}.
Qualitatively, the models produce samples of similar quality, comparing the truncated DSBM to RSBM.
For the animal faces, RSBM seems to better match the color of background between start and end image, and sometimes also the object itself. Together with the guaranteed domain support, this could explain the lower MSD.

\begin{table}[t]
    \caption{Evaluation metrics for the image-to-image syntheses. For \(\alpha\)-DSBM, the MSD is computed for the unclipped samples, whereas both FID and LPIPS require valid pixel values so the clipped samples are used. The columns show mean \(\pm\) one standard deviation computed over three seeds.}
    \label{tab:image-eval-scores}
    \centering
    \small
    \setlength{\tabcolsep}{4pt}
    \begin{tabular}{llccccc}
        \toprule
        & & \multicolumn{2}{c}{MNIST \(\leftrightarrow\) EMNIST} & \multicolumn{3}{c}{AFHQ (\(64 \times 64\))} \\
        \cmidrule(lr){3-4} \cmidrule(lr){5-7}
        & & MSD \textdownarrow & FID \textdownarrow & MSD \textdownarrow & LPIPS \textdownarrow & FID \textdownarrow \\
        \midrule
        \multirow{2}{*}{\(\Pi_0 \to \Pi_1\)} & \(\alpha\)-DSBM & \(0.103 \pm 0.008\) & \(\phantom{0}7.76 \pm 3.82\) & \(0.030 \pm 0.008\) & \(0.253 \pm 0.013\) & \(29.5 \pm 5.8\) \\
         & \(\alpha\)-RSBM & \(0.091 \pm 0.002\) & \(\phantom{0}6.53 \pm 1.43\) & \(0.022 \pm 0.005\) & \(0.242 \pm 0.018\) & \(31.3 \pm 2.1\) \\
        \cmidrule(l){1-2}
        \multirow{2}{*}{\(\Pi_0 \gets \Pi_1\)} & \(\alpha\)-DSBM & \(0.103 \pm 0.007\) & \(\phantom{0}9.48 \pm 3.33\) & \(0.045 \pm 0.009\) & \(0.301 \pm 0.019\) & \(37.5 \pm 3.9\) \\
         & \(\alpha\)-RSBM & \(0.092 \pm 0.002\) & \(\phantom{0}9.31 \pm 1.82\) & \(0.024 \pm 0.004\) & \(0.261 \pm 0.019\) & \(41.0 \pm 6.1\) \\
        \bottomrule
    \end{tabular}
\end{table}

\begin{figure}[t]
    \centering
    \begin{subfigure}{0.48\textwidth}
        \includegraphics[width=1\linewidth]{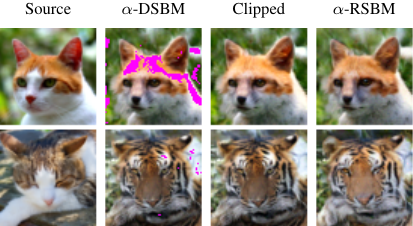}
        \caption{\texttt{cat} \(\to\) \texttt{wild}}
    \end{subfigure}
    \vspace{2mm}
    \hfill
    \begin{subfigure}{0.48\textwidth}
        \includegraphics[width=1\linewidth]{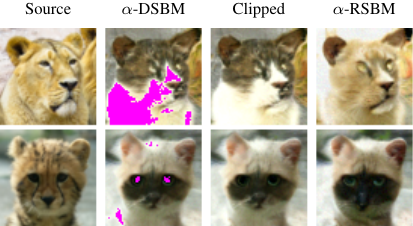}
        \caption{\texttt{wild} \(\to\) \texttt{cat}}
    \end{subfigure}
    \begin{subfigure}{0.48\textwidth}
        \includegraphics[width=1\linewidth]{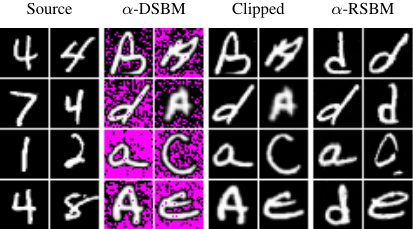}
        \caption{\texttt{MNIST} \(\to\) \texttt{EMNIST}}
    \end{subfigure}
    \hfill
    \begin{subfigure}{0.48\textwidth}
        \includegraphics[width=1\linewidth]{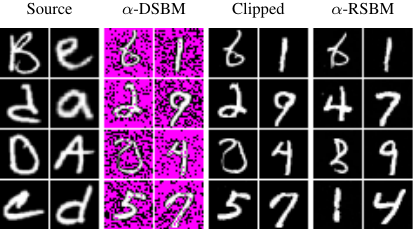}
        \caption{\texttt{EMNIST} \(\to\) \texttt{MNIST}}
    \end{subfigure}
    \caption{
        Transports between classes in AFHQ (\(64 \times 64\)) and MNIST/EMNIST.
        For each direction, the first column contains the initial (true) sample, followed by generated samples from \(\alpha\)-DSBM, clipped \(\alpha\)-DSBM and \(\alpha\)-RSBM.
        The out-of-bounds pixels in the \(\alpha\)-DSBM are colored magenta.
        The rows correspond to different samples picked at random from the test data.
    }
    \label{fig:afhq64-mnist-random-samples}
\end{figure}

\section{Discussion}

\vspace{-0.6em}
We have shown in this work that the IMF method of \citet{shi2023diffusion} can be practically applied in the setting of reflected DSBs, and that the theoretical properties are preserved.
This allows learning high-dimensional DSBs on constrained domains such as $[0, 1]^d$ with the partially simulation-free framework of flow-matching and rectified flows \citep{liu2022flow}, guaranteeing that the bounded supports of the distributions are maintained.
Importantly, we demonstrate that the overhead of our method is negligible compared to \citet{shi2023diffusion} and \citet{debortoli2024schrodinger}, effectively establishing our method as a cheap way of training reflected DSBs.

Extensions such as more complicated geometries and reference dynamics are left to future work, but conceptually, the main requirement is to implement efficient methods of computing the Doob $h$-transform term of the process, and marginally sampling its bridge processes.
In the best case, these are reachable by analytical formulas and sampling schemes, but in general cases one may instead attempt to rely on neural bridge samplers and $h$-transform approximators \citep{de2021simulating, baker2024score}.
Alternatively, we may rely on solving PDEs for the $h$-transforms as in \citet{lou2023reflected}.
\enlargethispage{\baselineskip}

\section*{Acknowledgements}
This work was partially supported by the Wallenberg AI, Autonomous Systems and Software Program
(WASP) funded by the Knut and Alice Wallenberg
Foundation.

Computational resources were provided by the National Academic Infrastructure for Supercomputing in Sweden (NAISS), funded by the Swedish Research Council,
and on the Berzelius system funded by the Knut and Alice Wallenberg Foundation and operated by NAISS.

\bibliographystyle{plainnat}
\bibliography{refs}

@article{bernton2022entropic,
  title={Entropic optimal transport: Geometry and large deviations},
  author={Bernton, Espen and Ghosal, Promit and Nutz, Marcel},
  journal={Duke Mathematical Journal},
  volume={171},
  number={16},
  pages={3363--3400},
  year={2022},
  publisher={Duke University Press}
}

@article{mikami2004monge,
  title={{Monge}’s problem with a quadratic cost by the zero-noise limit of h-path processes},
  author={Mikami, Toshio},
  journal={Probability Theory and Related Fields},
  volume={129},
  number={2},
  pages={245--260},
  year={2004},
  publisher={Springer}
}

@book{pilipenko2014introduction,
  title={An introduction to stochastic differential equations with reflection},
  author={Pilipenko, Andrey},
  year={2014},
  publisher={Universit{\"a}tsverlag Potsdam}
}

@inproceedings{caluya2021reflected,
  title={Reflected {Schr{\"o}dinger} bridge: Density control with path constraints},
  author={Caluya, Kenneth F and Halder, Abhishek},
  booktitle={2021 American Control Conference (ACC)},
  pages={1137--1142},
  year={2021},
  organization={IEEE}
}

@article{leonard2012schrodinger,
  title={From the {Schr{\"o}dinger} problem to the {Monge--Kantorovich} problem},
  author={L{\'e}onard, Christian},
  journal={Journal of Functional Analysis},
  volume={262},
  number={4},
  pages={1879--1920},
  year={2012},
  publisher={Elsevier}
}

@article{leonard2014survey,
  title={A survey of the {Schr\"{o}dinger} problem and some of its connections with optimal transport},
  author={L{\'e}onard, Christian},
  journal={Discrete Contin. Dyn. Syst.-A},
  volume={34},
  number={4},
  pages={1533--1574},
  year={2014}
}

@article{song2019generative,
  title={Generative modeling by estimating gradients of the data distribution},
  author={Song, Yang and Ermon, Stefano},
  journal={Advances in {Neural Information Processing Systems}},
  volume={32},
  year={2019}
}

@inproceedings{sohl2015deep,
  title={Deep unsupervised learning using nonequilibrium thermodynamics},
  author={Sohl-Dickstein, Jascha and Weiss, Eric and Maheswaranathan, Niru and Ganguli, Surya},
  booktitle={Proceedings of the 32nd International Conference on Machine Learning},
  pages={2256--2265},
  year={2015},
  organization={{PMLR}}
}

@article{ho2020denoising,
  title={Denoising diffusion probabilistic models},
  author={Ho, Jonathan and Jain, Ajay and Abbeel, Pieter},
  journal={Advances in {Neural Information Processing Systems}},
  volume={33},
  pages={6840--6851},
  year={2020}
}

@article{song2020score,
  title={Score-based generative modeling through stochastic differential equations},
  author={Song, Yang and Sohl-Dickstein, Jascha and Kingma, Diederik P and Kumar, Abhishek and Ermon, Stefano and Poole, Ben},
  journal={International Conference on Learning Representations},
  year={2021}
}

@article{bortoli2021diffusion,
  title={Diffusion {Schr{\"o}dinger} bridge with applications to score-based generative modeling},
  author={De Bortoli, Valentin and Thornton, James and Heng, Jeremy and Doucet, Arnaud},
  journal={Advances in {Neural Information Processing Systems}},
  volume={34},
  pages={17695--17709},
  year={2021}
}

@inproceedings{lipman2023flow,
title={Flow Matching for Generative Modeling},
author={Lipman, Yaron and Chen, Ricky T. Q. and Ben-Hamu, Heli and Nickel, Maximilian and Le, Matthew},
booktitle={International Conference on Learning Representations},
year={2023},
}

@article{liu2022flow,
  title={Flow straight and fast: Learning to generate and transfer data with rectified flow},
  author={Liu, Xingchao and Gong, Chengyue and Liu, Qiang},
  journal={International Conference on Learning Representations},
  year={2023}
}

@article{shi2023diffusion,
  title={Diffusion {Schr{\"o}dinger} bridge matching},
  author={Shi, Yuyang and De Bortoli, Valentin and Campbell, Andrew and Doucet, Arnaud},
  journal={Advances in {Neural Information Processing Systems}},
  volume={36},
  pages={62183--62223},
  year={2023}
}

@InProceedings{deng2024reflected,
  title={Reflected {Schr\"{o}dinger} Bridge for Constrained Generative Modeling},
  author={Deng, Wei and Chen, Yu and Yang, Nicole Tianjiao and Du, Hengrong and Feng, Qi and Chen, Ricky T. Q.},
  booktitle = 	 {Proceedings of the Fortieth Conference on Uncertainty in Artificial Intelligence},
  pages = 	 {1055--1082},
  year = 	 {2024},
 publisher =    {{PMLR}},
}

@inproceedings{lou2023reflected,
  title={Reflected diffusion models},
  author={Lou, Aaron and Ermon, Stefano},
  booktitle={Proceedings of the 40th International Conference on Machine Learning},
  pages={22675--22701},
  year={2023},
  organization={{PMLR}}
}

@article{chen2021likelihood,
  title={Likelihood training of {Schr\"{o}dinger} bridge using forward-backward {SDEs} theory},
  author={Chen, Tianrong and Liu, Guan-Horng and Theodorou, Evangelos A},
    journal={International Conference on Learning Representations},
  year={2022}
}

@article{cuturi2013,
  title={Sinkhorn distances: Lightspeed computation of optimal transport},
  author={Cuturi, Marco},
  journal={Advances in {Neural Information Processing Systems}},
  volume={26},
  year={2013}
}

@article{nilsson2025large,
  title={Large deviations for scaled families of {Schr\"odinger} bridges with reflection},
  author={Nilsson, Viktor and Nyquist, Pierre},
  journal={arXiv preprint arXiv:2506.03999},
  year={2025}
}

@book{rogers2000diffusions,
  title={Diffusions, {Markov} processes, and martingales},
  author={Rogers, L. Chris G. and Williams, David},
  volume={2},
  year={2000},
  publisher={Cambridge University Press}
}

@article{debortoli2024schrodinger,
  title={{Schr\"odinger} bridge flow for unpaired data translation},
  author={De Bortoli, Valentin and Korshunova, Iryna and Mnih, Andriy and Doucet, Arnaud},
  journal={Advances in {Neural Information Processing Systems}},
  volume={37},
  pages={103384--103441},
  year={2024}
}

@article{dhariwal2021diffusion,
  title={Diffusion models beat {GANs} on image synthesis},
  author={Dhariwal, Prafulla and Nichol, Alexander},
  journal={Advances in {Neural Information Processing Systems}},
  volume={34},
  pages={8780--8794},
  year={2021}
}

@article{karras2022elucidating,
  title={Elucidating the design space of diffusion-based generative models},
  author={Karras, Tero and Aittala, Miika and Aila, Timo and Laine, Samuli},
  journal={Advances in {Neural Information Processing Systems}},
  volume={35},
  pages={26565--26577},
  year={2022}
}

@article{albergo2023stochastic,
  title={Stochastic interpolants: A unifying framework for flows and diffusions},
  author={Albergo, Michael S and Boffi, Nicholas M and Vanden-Eijnden, Eric},
  journal={arXiv preprint arXiv:2303.08797},
  year={2023}
}

@article{liu2022rectified,
  title={Rectified flow: A marginal preserving approach to optimal transport},
  author={Liu, Qiang},
  journal={arXiv preprint arXiv:2209.14577},
  year={2022}
}

@article{peluchetti2023diffusion,
  title={Diffusion bridge mixture transports, {Schr{\"o}dinger} bridge problems and generative modeling},
  author={Peluchetti, Stefano},
  journal={Journal of Machine Learning Research},
  volume={24},
  number={374},
  pages={1--51},
  year={2023}
}

@Book{Schuss2013,
  author           = {Schuss, Zeev},
  year             = {2013},
  title            = {Brownian Dynamics at Boundaries and Interfaces: In Physics, Chemistry, and Biology},
  doi              = {10.1007/978-1-4614-7687-0},
  isbn             = {9781461476870},
  publisher        = {Springer New York},
  issn             = {2196-968X},
  journaltitle     = {Applied Mathematical Sciences},
}

@incollection{leonard2012girsanov,
  title={Girsanov theory under a finite entropy condition},
  author={L{\'e}onard, Christian},
  booktitle={S{\'e}minaire de Probabilit{\'e}s XLIV},
  pages={429--465},
  year={2012},
  publisher={Springer}
}

@inproceedings{perez2018FiLM,
  title={{FiLM}: Visual reasoning with a general conditioning layer},
  author={Perez, Ethan and Strub, Florian and de Vries, Harm and Dumoulin, Vincent and Courville, Aaron},
  booktitle={Proceedings of the {AAAI} Conference on Artificial Intelligence},
  year={2018}
}

@article{song2020improved,
  title={Improved techniques for training score-based generative models},
  author={Song, Yang and Ermon, Stefano},
  journal={Advances in {Neural Information Processing Systems}},
  volume={33},
  pages={12438--12448},
  year={2020}
}

@article{cattiaux1988time,
  title={Time reversal of diffusion processes with a boundary condition},
  author={Cattiaux, Patrick},
  journal={Stochastic Processes and Their Applications},
  volume={28},
  number={2},
  pages={275--292},
  year={1988},
  publisher={Elsevier}
}

@article{ho2022classifier,
  title={Classifier-free diffusion guidance},
  author={Ho, Jonathan and Salimans, Tim},
  journal={arXiv preprint arXiv:2207.12598},
  year={2022}
}

@misc{hajotthu_oresund_bridge_2015,
  author       = {{Hajotthu}},
  title        = {{Öresundbrücke nach Kopenhagen}},
  year         = {2015},
  month        = {May},
  day          = {29},
  howpublished = {Wikimedia Commons},
  url          = {https://commons.wikimedia.org/wiki/File:%C3%96resundbr%C3%BCcke_nach_Kopenhagen.JPG},
  urldate      = {2026-05-06},
  note         = {Licensed under Creative Commons Attribution-ShareAlike 4.0 International (CC BY-SA 4.0)}
}

@article{kingma2014adam,
  title={Adam: A method for stochastic optimization},
  author={Kingma, Diederik P and Ba, Jimmy},
  journal={arXiv preprint arXiv:1412.6980},
  year={2014}
}

@misc{mnist,
  author = {LeCun, Yann and Cortes, Corinna and Burges, Christopher J. C.},
  title = {{The MNIST Database of Handwritten Digits}},
  year = 1994,
  howpublished = {\url{http://yann.lecun.com/exdb/mnist/}},
}

@inproceedings{emnist,
  title={{EMNIST}: Extending {MNIST} to handwritten letters},
  author={Cohen, Gregory and Afshar, Saeed and Tapson, Jonathan and van Schaik, Andre},
  booktitle={2017 International Joint Conference on Neural Networks ({IJCNN})},
  pages={2921--2926},
  year={2017},
  organization={{IEEE}}
}

@inproceedings{choi2020starganv2,
  title={{StarGAN} v2: Diverse Image Synthesis for Multiple Domains},
  author={Yunjey Choi and Youngjung Uh and Jaejun Yoo and Jung-Woo Ha},
  booktitle={Proceedings of the {IEEE} Conference on Computer Vision and Pattern Recognition},
  year={2020}
}

@inproceedings{lpips,
  title={The unreasonable effectiveness of deep features as a perceptual metric},
  author={Zhang, Richard and Isola, Phillip and Efros, Alexei A and Shechtman, Eli and Wang, Oliver},
  booktitle={Proceedings of the {IEEE} Conference on Computer Vision and Pattern Recognition},
  pages={586--595},
  year={2018}
}

@article{de2021simulating,
    author = {Heng, Jeremy and De Bortoli, Valentin and Doucet, Arnaud and Thornton, James},
    title = {Simulating diffusion bridges with score matching},
    journal = {Biometrika},
    volume = {112},
    number = {4},
    year = {2025},
    doi = {10.1093/biomet/asaf048},
}

@article{baker2024score,
  title={Score matching for bridges without learning time-reversals},
  author={Baker, Elizabeth L and Schauer, Moritz and Sommer, Stefan},
  journal={arXiv preprint arXiv:2407.15455},
  year={2024}
}

@article{paszke2019pytorch,
  title={{PyTorch}: An imperative style, high-performance deep learning library},
  author={Paszke, Adam and Gross, Sam and Massa, Francisco and Lerer, Adam and Bradbury, James and Chanan, Gregory and Killeen, Trevor and Lin, Zeming and Gimelshein, Natalia and Antiga, Luca and others},
  journal={Advances in {Neural Information Processing Systems}},
  volume={32},
  year={2019}
}

@article{palmowski2002technique,
  title={A Technique for Exponential Change of Measure for {Markov} Processes},
  author={Palmowski, Zbigniew and Rolski, Tomasz},
  journal={Bernoulli},
  volume={8},
  number={6},
  pages={767--785},
  year={2002},
}

@book{legall2016,
  title={Brownian Motion, Martingales, and Stochastic Calculus},
  author={Le Gall, Jean-Fran{\c{c}}ois},
  year={2016},
  publisher={Springer}
}

@article{bjork2015pedestriansguidelocaltime,
      title={The Pedestrian's Guide to Local Time}, 
      author={Bj{\"o}rk, Tomas},
      year={2015},
      journal={arXiv preprint arXiv:1512.08912},
}

@article{fishman2023diffusion,
  title   = {Diffusion Models for Constrained Domains},
  author  = {Fishman, Nic and Klarner, Leo and De Bortoli, Valentin and Mathieu, Emile and Hutchinson, Michael},
  journal = {Transactions on Machine Learning Research},
  year    = {2023},
}

@article{fishman2023metropolis,
  title={Metropolis sampling for constrained diffusion models},
  author={Fishman, Nic and Klarner, Leo and Mathieu, Emile and Hutchinson, Michael and De Bortoli, Valentin},
  journal={Advances in {Neural Information Processing Systems}},
  volume={36},
  pages={62296--62331},
  year={2023}
}

@article{liu2023mirror,
  title={Mirror diffusion models for constrained and watermarked generation},
  author={Liu, Guan-Horng and Chen, Tianrong and Theodorou, Evangelos and Tao, Molei},
  journal={Advances in {Neural Information Processing Systems}},
  volume={36},
  pages={42898--42917},
  year={2023}
}

@article{christopher2024constrained,
  title={Constrained synthesis with projected diffusion models},
  author={Christopher, Jacob K and Baek, Stephen and Fioretto, Ferdinando},
  journal={Advances in {Neural Information Processing Systems}},
  volume={37},
  pages={89307--89333},
  year={2024}
}

@inproceedings{liu2023learning,
title={Learning Diffusion Bridges on Constrained Domains},
author={Liu, Xingchao and Wu, Lemeng and Ye, Mao and Liu, Qiang},
booktitle={International Conference on Learning Representations},
year={2023},
}

@article{hertrich2026relation,
  title={On the relation between rectified flows and optimal transport},
  author={Hertrich, Johannes and Chambolle, Antonin and Delon, Julie},
  journal={Advances in Neural Information Processing Systems},
  volume={38},
  pages={172578--172610},
  year={2026}
}

@article{gobet2001euler,
  title={Euler schemes and half-space approximation for the simulation of diffusion in a domain},
  author={Gobet, Emmanuel},
  journal={ESAIM: Probability and Statistics},
  volume={5},
  pages={261--297},
  year={2001},
  publisher={EDP Sciences}
}

@misc{nordenhog2025,
      title={Score-based constrained generative modeling via {L}angevin diffusions with boundary conditions}, 
      author={Nordenh{\"o}g, Adam and Sharma, Akash},
      year={2025},
      eprint={2510.23985},
      archivePrefix={arXiv},
      primaryClass={stat.ML},
      url={https://arxiv.org/abs/2510.23985}, 
}

@inproceedings{adamw,
    title={Decoupled Weight Decay Regularization},
    author={Loshchilov, Ilya and Hutter, Frank},
    booktitle={{I}nternational {C}onference on {L}earning {R}epresentations},
    year={2019},
    url={https://openreview.net/forum?id=Bkg6RiCqY7},
}

\clearpage
\appendix
\section{Theory and proofs}
\label{sec:appendix-proofs}
In this appendix, we prove Proposition~\ref{prop:markov-proj-refl-main}.
The proof is essentially a straightforward extension of that of \citet{shi2023diffusion}, but with a difference in treatment of the Doob h-transform.

We drop the superscript of the reflected path measures \(\bQ^r\) and the Markov class \(\calM^r\) from the main text.
Recall the reflected SDE \eqref{eq:def-Qr-SDE}:
\[dX_t = \mu(t, X_t) dt + \sigma(t, X_t) dW_t + n(X_t) dL_t, \quad X_0 \sim \bQ_0,\]
on a domain \(D \subseteq \bR^d\) with normal reflections.\footnote{We here allow \(\sigma\) to be a scalar-valued function of \(x\) in addition to \(t\). The dependency on \(x\) is omitted in the main text to preserve the problem symmetry when going from \(\pi_1\) to \(\pi_0\); in general an additional \(\nabla \sigma^2\) drift term appears in the reverse direction.}
Assume that the SDE has a weak solution under some filtered probability space \((\Omega, \calF, (\calF_t)_{t\in[0,1]}, \bP)\).
Let \(\bQ \coloneq X_\# \bP\) be its path measure.

For a given distribution \(\Pi_{0, 1}\) on \(D \times D\) we construct a path measure in the reciprocal class as \(\Pi = \bQ_{\mid 0, 1} \Pi_{0, 1}\). Assume \(\Pi_{0, 1}\) and \(\bQ_{0, 1}\) are equivalent, i.e., mutually absolutely continuous.

\subsection{Doob's h-transform for reflected bridge}
A central notion in the Markovian projection is that Doob's h-transform provides the means of constructing a Markov process from another Markov process which has been conditioned on ending (at time \(t=1\)) in a certain point.
A typical example is that of the Brownian bridge.

More generally, our aim is to, conditioned on \(\calF_0\), reweight the terminal distribution of \(\bQ\) so that it becomes \(\Pi_{1 \mid 0}\) while preserving the bridge distribution \(\bQ_{\mid 0, 1}\).
\citet{shi2023diffusion} achieve this via the infinitesimal generator, identifying the change in drift from the generator of the process of \(\bQ_{\mid 0}\) under the new measure \citep[see][]{palmowski2002technique}.
To avoid characterizing the domain of the new generator in the reflected case, we instead use a Girsanov approach.

Define on \(D \times D\)
\[\varphi_{1 \mid 0}(x_1 \mid x_0) \coloneq \frac{d\Pi_{1 \mid 0}}{d \bQ_{1 \mid 0}}(x_1 \mid x_0),\]
and note that
\[\frac{d \Pi_{\mid 0}}{d \bQ_{\mid 0}} = \varphi_{1 \mid 0}.\]
From this we create our martingale that allows us to use Girsanov's theorem:
\[\frac{d \Pi_{\mid 0} \vert_{\calF_t}}{d\bQ_{\mid 0}\vert_{\calF_t}} = \bE_{\bQ_{\mid 0}}\left[\frac{d \Pi_{1 \mid 0} }{d\bQ_{1 \mid 0}} (X_1 \mid x_0) \;\middle| \; \calF_t\right] = \bE_{\bQ_{1 \mid t}}[\varphi_{1 \mid 0}(X_1 \mid x_0) \mid X_t] \eqcolon \varphi_{t \mid 0}(X_t \mid x_0),\]
where the second equality holds since \(\bQ\) is Markov.
Also note that \(\varphi_{0 \mid 0} = 1\).

Denoting \(d\tilde{\bP}_{\mid 0} \coloneq \varphi_{1 \mid 0}(X_1 \mid x_0) d \bP_{\mid 0}\), we have the following result.
\begin{proposition}
    \label{prop:Pi-given-0-RSDE}
    Assume \((t, x) \mapsto \varphi_{t \mid 0}(x \mid x_0)\) is \(C^{1, 2}\) on \([0,1] \times D\) for all \(x_0 \in D\).
    Assuming also that the solution exists, then \(\Pi_{\mid 0}\) is Markov, corresponding to the process
    \begin{equation}
        \label{eq:sde-of-Pi-given-0}
        \begin{split}
            X_t = X_0 &+ \int_0^t \left(\mu(s, X_s) + \sigma^2(s, X_s) \bE_{\Pi_{1 \mid 0, s}}[\nabla \log \bQ^r_{1\mid s}(X_1 \mid X_s) \mid X_0, X_s]\right) ds \\
            &+ \int_0^t \sigma(s, X_s) d\widetilde{W}_s + \int_0^t n(X_s) dL_s,
        \end{split}
    \end{equation}
    where \(\widetilde{W}\) is a Brownian motion under \(\tilde{\bP}_{\mid 0}\).
\end{proposition}
\begin{proof}
    Denote by \(\calA\) the differential operator
    \[(\calA f)(t, x) \coloneq \frac{\partial f(t, x)}{\partial t} + \mu(t, x) \cdot \nabla f(t, x) + \frac{1}{2} \sigma^2 (t, x) \Delta f(t, x).\]

    Fix \(x_0\). Applying Itô's lemma to the function \((t, x) \mapsto \varphi_{t \mid 0}(x \mid x_0)\) yields
    \begin{equation*}
        \begin{split}
            \varphi_{t \mid 0}(X_t \mid x_0) 
            = 1
              &+ \int_0^t \calA \varphi_{s \mid 0}(X_s \mid x_0) ds 
              + \int_0^t \nabla \varphi_{s \mid 0}(X_s \mid x_0)\T \sigma(s, X_s) dW_s \\
              &+ \int_0^t \nabla \varphi_{s \mid 0}(X_s \mid x_0) \T n(X_s) dL_s.
        \end{split}
    \end{equation*}
    Rearranging the terms as
    \begin{equation*}
        \begin{split}
            \varphi_{t \mid 0}(X_t \mid x_0) - 1 - \int_0^t \nabla \varphi_{s \mid 0}(X_s &\mid x_0)\T \sigma(s, X_s) dW_s
            \\ &= \int_0^t \calA \varphi_{s \mid 0}(X_s \mid x_0) ds
              + \int_0^t \nabla \varphi_{s \mid 0}(X_s \mid x_0) \T n(X_s) dL_s,
        \end{split}
    \end{equation*}
    since the terms on the left-hand side are continuous martingales, the right-hand side is also a continuous martingale.
    But this has finite variation so it must be zero almost surely \citep[see e.g.][Theorem~4.8]{legall2016}.
    Hence,
    \begin{equation*}
        \begin{split}
            d \varphi_{t \mid 0}(X_t \mid x_0) 
            & = \nabla \varphi_{t \mid 0}(X_t \mid x_0)\T \sigma(t, X_t) dW_t \\
            & = \varphi_{t \mid 0}(X_t \mid x_0) \nabla \log \varphi_{t \mid 0}(X_t \mid x_0)\T \sigma(t, X_t) dW_t,
        \end{split}
    \end{equation*}
    where equivalence between \(\Pi_{\mid 0}\) and \(\bQ_{\mid 0}\) assures \(\varphi_{t \mid 0} > 0\) a.s. By Girsanov's theorem,
    \[\widetilde{W}_t \coloneq W_t - \int_0^t \sigma (s, X_s) \nabla \log \varphi_{s \mid 0}(X_s \mid x_0) ds\]
    is a Brownian motion on \([0, 1]\) under \(\tilde{\bP}_{\mid 0}\).
    Therefore, it holds that
    \[dX_t = \left(\mu(t, X_t) + \sigma^2 (t, X_t) \nabla \log \varphi_{t \mid 0}(X_t \mid x_0) \right) dt + \sigma(t, X_t) d\widetilde{W}_t + n(X_t) dL_t.\]
    Finally, the computation in \citet[Proof of Prop.~2]{shi2023diffusion} establishing that
    \[\nabla \log \varphi_{t \mid 0}(x_t \mid x_0) = \bE_{\Pi_{1\mid 0, t}}[\nabla \log \bQ_{1 \mid t}(X_1 \mid X_t) \mid X_0 = x_0, X_t = x_t]\]
    follows analogously in the reflected setting, resulting in \eqref{eq:sde-of-Pi-given-0}.
\end{proof}

\subsection{Mixture of diffusion bridges}
The last ingredient needed to prove Proposition~\ref{prop:markov-proj-refl-main} is a means of getting the dynamics of the \(\Pi_0\)-mixture of diffusion bridges \(\Pi_{\mid 0}\) given by \eqref{eq:sde-of-Pi-given-0}.
This is given by an extension of \citet[Theorem~1]{peluchetti2023diffusion} to reflected SDEs.
For completeness, the theorem is restated here in our notation before we provide the extension to our setting.
\begin{theorem}[{\citet[Theorem~1]{peluchetti2023diffusion}}]
    \label{thm:peluchetti-mixture-fp}
    Consider the set of SDEs indexed by $\lambda \in \Lambda$
    \begin{equation}
    \label{eq:SDEs-indexed-by-lambda}
        \begin{split}
            & dX_t^\lambda = \mu^\lambda(t, X_t^\lambda) dt + \sigma^\lambda(t, X_t^\lambda) dW_t^\lambda, \quad t \in [0, 1], \\
            & X_0^\lambda \sim \bP_0^\lambda,
        \end{split}
    \end{equation}
    corresponding to the path measures $\{\bP^\lambda\}_{\lambda \in \Lambda}$ and with marginal densities $p_t^\lambda$. For a mixing $\Psi$ on $\Lambda$, let $\bP$ be obtained by taking the $\Psi$-mixture of \eqref{eq:SDEs-indexed-by-lambda} over $\lambda \in \Lambda$. In particular, define the mixture marginal densities $p_t, t \in (0,1)$, and the mixture initial measure $\bP_0$ by
    \begin{equation}
        p_t(x) \coloneq \int_\Lambda p_t^\lambda(x) \Psi(d\lambda), \quad \bP_0(\cdot) \coloneq \int_\Lambda \bP_0^\lambda(\cdot) \Psi(d\lambda).
    \end{equation}
    Consider the SDE
    \begin{equation}
        \label{eq:peluchetti-X_t}
        \begin{split}
            & dX_t = \mu(t, X_t) dt + \sigma(t, X_t) dW_t, \quad t \in [0, 1], \\
            & \mu(t, x) \coloneq \frac{1}{p_t(x)} \int_\Lambda \mu^\lambda(t, x) p_t^\lambda(x) \Psi(d\lambda), \\
            & \sigma(t, x) \coloneq \frac{1}{p_t(x)} \int_\Lambda \sigma^\lambda(t, x) p_t^\lambda(x) \Psi(d\lambda), \\
            & X_0 \sim \bP_0,
        \end{split}
    \end{equation}
    with law $\tilde{\bP}$. Then, under mild conditions\footnote{See \citep{peluchetti2023diffusion} for details.}, 
    \begin{equation}
        \label{eq:Ptilde_t-equals-P_t}
        \tilde{\bP}_t = \bP_t \quad \forall t \in [0, 1].
    \end{equation}
\end{theorem}
\begin{remark}
    In the context of this paper, the set of SDEs is indexed by $x_0$, $\bP_0^{x_0} = \delta_{x_0}$ is a Dirac point mass, and mixing is done with $\Pi_0$, whereby $\bP^{x_0} = \Pi_{\mid 0}(\cdot \mid X_0 = x_0)$ and $\bP = \Pi$.
\end{remark}
\begin{corollary}
    \label{cor:law-of-mixture-of-reflected-bridges}
    Let $X^\lambda$ and $X$ be the corresponding reflected processes of \eqref{eq:SDEs-indexed-by-lambda} and \eqref{eq:peluchetti-X_t}. If $\sigma^\lambda(t, x) = \sigma(t, x)$ is the same for all $\lambda$,
    then \eqref{eq:Ptilde_t-equals-P_t} still holds.
\end{corollary}
\begin{proof}
    The non-reflected case is proved by showing $p$ and $\tilde{p}$, the marginal densities of $\bP$ and $\tilde{\bP}$, respectively, obey the same Fokker-Planck equation. The Fokker-Planck equation for a reflected process in a domain $D$ with normal $n$ is the same PDE inside $D$, but with a Neumann boundary condition \citep{Schuss2013}:
    \begin{equation}
        (\mu p_t - \frac{1}{2}\nabla_x (\sigma^2 p_t) ) \cdot n = 0, \quad x \in \partial D, \; t > 0.
    \end{equation}
    Since the Fokker-Planck equations coincide for a reflected and a non-reflected process inside $D$, it remains to check the boundary condition: From the PDE of $p_t$ we get the flux
    \begin{equation}
        J = \mu p_t - \frac{1}{2}\nabla_x (\sigma^2 p_t)
    \end{equation}
    where $\mu$ is as in \eqref{eq:peluchetti-X_t}. The Neumann boundary condition of $p_t$ then reads
    \begin{equation}
        \begin{split}
            J \cdot n 
            &= \left[ \mu p_t - \frac{1}{2}\nabla_x (\sigma^2 p_t) \right] \cdot n \\
            &= \left[ \int_\Lambda \mu^\lambda p_t^\lambda \Psi(d\lambda) - \frac{1}{2}\nabla_x  \left(\sigma^2 \int_\Lambda p_t^\lambda \Psi(d\lambda)\right) \right] \cdot n \\
            &= \left[ \int_\Lambda \mu^\lambda p_t^\lambda \Psi(d\lambda) -  \int_\Lambda \frac{1}{2}\nabla_x  (\sigma^2 p_t^\lambda) \Psi(d\lambda) \right] \cdot n \\
            &= \int_\Lambda \left[ \mu^\lambda p_t^\lambda - \frac{1}{2}\nabla_x (\sigma^2 p_t^\lambda) \right] \cdot n \,\Psi(d\lambda) \\
            &= 0,
        \end{split}
    \end{equation}
    where, as in Theorem~\ref{thm:peluchetti-mixture-fp}, we rely on the assumption that the exchange of limits going from the second to the third line holds.
\end{proof}

\subsection{Proving Proposition~\ref{prop:markov-proj-refl-main}}
\begin{proof}
    The procedure largely follows the proof of Proposition~2 in \citet{shi2023diffusion}.
    By proposition~\ref{prop:Pi-given-0-RSDE}, \(\Pi_{\mid 0}\) is the law of the RSDE
    \begin{equation*}
        \begin{split}
            dX_t^{x_0} &= \left(\mu(t, X_t^{x_0}) + \sigma^2 (t, X_t^{x_0}) \, \bE_{\Pi_{1 \mid 0, t}}[\nabla \log \bQ^r_{1\mid t}(X_1 \mid X_t) \mid X_0=x_0, X_t = X_t^{x_0}]\right) dt \\
            &\phantom{=.} + \sigma (t, X_t^{x_0}) dW_t + n(X_t^{x_0}) dL_t, \\
            X_0^{x_0} &= x_0.
        \end{split}
    \end{equation*}
    
    Mixing with \(\Pi_0\) gives \(\Pi_{\mid 0} \Pi_0 = \Pi\) which by Corollary~\ref{cor:law-of-mixture-of-reflected-bridges} has the same marginals as the RSDE with drift
    \begin{equation*}
        \begin{split}
            &\bE_{\Pi_{0 \mid t}} \left[\mu(t, X_t) + \sigma^2(t, X_t) \, \bE_{\Pi_{1 \mid 0, t}}[\nabla \log \bQ^r_{1\mid t}(X_1 \mid X_t) \mid X_0, X_t] \; \middle| \; X_t = x \right] \\
            &= \mu(t, x) + \sigma^2(t, x) \, \bE_{\Pi_{1 \mid t}}[\nabla \log \bQ^r_{1\mid t}(X_1 \mid X_t) \mid X_t = x]
        \end{split}
    \end{equation*}
    and diffusion \(\sigma\), which we recognize as the coefficients in \eqref{eq:def-sde-of-refl-markov-proj} of \(\bM^*\).
    This establishes (i).

    For (ii) and (iii), first we compute an expression of the KL divergence between \(\Pi\) and \(\bM \in \calM\), the path law of the solution to the RSDE
    \[X_t = X_0 + \int_0^t \left(\mu(s, X_s) + v^{\bM}_s(X_s)\right) ds + \int_0^t \sigma^{\bM}(s, X_s) dW^{\bM}_s + \int_0^t n(X_s) dL_s, \quad X_0 \sim \bM_0.\]
    Assume \(\KL (\Pi \mid\mid \bM) < \infty\).
    Then the quadratic variation must be the same for both processes, \((\sigma^{\bM})^2 = \sigma^2\), almost surely.
    Since the local time term has bounded variation, we may use Girsanov's theorem of \citet[Theorem~2.1]{leonard2012girsanov} to identify the change in drift going from \(\bM\) to \(\Pi\) by 
    \begin{equation*}
        \begin{split}
            \int_0^t \sigma^2(s, X_s) \beta_s ds = &\int_0^t \left(\sigma^2(s, X_s) \, \bE_{\Pi_{1 \mid 0, s}}[\nabla \log \bQ^r_{1\mid s}(X_1 \mid X_s) \mid X_0, X_s] - v_s(X_s)\right) ds.
        \end{split}
    \end{equation*}
    The key to this step is to note that the reflection term does not change.
    Thus, we have
    \[\beta_t = \sigma^{-2}(t, X_t)\left( u_t(X_0, X_t) - v_t(X_t) \right),\]
    with
    \[u_t(X_0, X_t) \coloneq \sigma^2(t, X_t)\,\bE_{\Pi_{1 \mid 0, t}}[\nabla \log \bQ^r_{1\mid t}(X_1 \mid X_t) \mid X_0, X_t].\]
    Using \citep[Theorem~2.3]{leonard2012girsanov}, the expression of the KL is the same form as in the unreflected case in \citet{shi2023diffusion}:
    \begin{equation*}
        \begin{split}
            \KL(\Pi \mid\mid &\bM) - \KL(\Pi_0 \mid\mid \bM_0) \\
            &= \frac{1}{2} \bE_\Pi \left[\int_0^1 \sigma^{-2}(t, X_t)\|u_t(X_0, X_t) - v_t(X_t)\|^2 dt\right] \\
            &= \frac{1}{2} \int_0^1 \bE_{\Pi_{0, t}} \left[ \sigma^{-2}(t, X_t)\|u_t(X_0, X_t) - v_t(X_t)\|^2 \right] dt,
        \end{split}
    \end{equation*}
    where the change in order of integration is valid by Fubini's theorem.
    The KL divergence is minimized when \(\bM_0 = \Pi_0\) and for each \(t\),
    \begin{equation*}
        \begin{split}
            v^*_t(x) &= \bE_{\Pi_{0 \mid t}}\left[u_t(X_0, X_t) \mid X_t = x \right] \\
            &= \bE_{\Pi_{0 \mid t}}\left[\sigma^2(t, X_t)\,\bE_{\Pi_{1 \mid 0, t}}[\nabla \log \bQ^r_{1\mid t}(X_1 \mid X_t) \mid X_0, X_t] \; \middle| \; X_t = x\right] \\
            &= \sigma^2(t, x)\, \bE_{\Pi_{1 \mid t}}[\nabla \log \bQ^r_{1\mid t}(X_1 \mid X_t) \mid X_t = x],
        \end{split}
    \end{equation*}
    i.e. the h-transform term in the drift of \eqref{eq:def-sde-of-refl-markov-proj}.
\end{proof}

\section{Experiment and model details}
\label{sec:appendix-training-details}

\subsection{Toy 2D problems}
The primary goal of these experiments is to compare RSBM and RSB-FBSDE under a limited wall-time training budget.
\citet{deng2024reflected} do not provide any 2D benchmark to compare with, so we attempt to set up these experiments in favor of RSB-FBSDE.
Although initially planned to allot one hour to each method, RSBM is given one and RSB-FBSDE is given two hours of training.
This is to allow RSB-FBSDE more time to converge for purpose of making the sweep more meaningful.
We base the architecture and hyperparameter choices of our experiments on the implementation of RSB-FBSDE from \url{https://github.com/WayneDW/Reflected_Schrodinger_Bridge}.
This implementation does not supply any experiments with rectangular domain or plain Brownian motion as reference process, so we create the uniform-to-checkerboard experiment.
It is presented in \cref{tab:hyperparameters-2D}, with the experiment used as reference in the upstream.

A sweep was performed for each method.
The selection criteria is picking the model configuration in the sweep with lowest control cost, averaged over seed and mean, among those satisfying a marginal constraint proxy.
This proxy is determined as follows:
At each marginal, true samples are generated and \(\calW_2\) estimated for the empirical distributions, defining the number of samples drawn as the evaluation batch size.
This is done multiple times, creating multiple estimates of \(\calW_2\).
These are in turn randomly grouped in threes and an average is computed, mirroring taking the average over the seeds.
The threshold defining the proxy is taken as the 95th percentile of these averages.
A model is classified as satisfying the marginal constraints if the seed-average empirical \(\calW_2\) estimated from an evaluation batch is below the threshold going in both directions.
If no parameter configuration in the sweep satisfies this, which turned out to be the case for RSB-FBSDE, the threshold is, per direction, set to the best achievable level by any configuration \textit{plus} \(2/\sqrt{3}\) seed standard deviations.

\begin{table}[t]
    \centering
    \begin{threeparttable}
        \caption{Uniform-to-checkerboard RSB-FBSDE experiment parameters. Parameters included in the sweep are written as sets, with the final configuration in bold.}
        \label{tab:hyperparameters-2D}
        \setlength{\tabcolsep}{5pt}
        \renewcommand{\arraystretch}{1.12}
        \renewcommand{\tabularxcolumn}[1]{m{#1}}
        \begin{tabularx}{\linewidth}{@{}>{\raggedright\arraybackslash}m{0.30\linewidth}
                                      *{2}{>{\raggedright\arraybackslash}X}@{}}
            \toprule
            Setting & This work & Upstream\tnote{a} \\
            \midrule
            Domain & Square \([-2\pi,2\pi]^2\) & Flower, radius range \(4\)--\(8\) \\
            \(\pi_0\) & Uniform on the square & Standard Gaussian truncated to the flower \\
            \(\pi_1\) & \(4\times4\) checkerboard on the square & \(4\times4\) checkerboard truncated to the flower \\
            Reference process & Reflected Brownian motion, \(\sigma=0.5\) &
                Reflected VP diffusion \\
            Quadratic variation\tnote{b} & \(0.25\) & \(0.27\) \\
            Euler--Maruyama steps & 100 & 100 \\
            Network & Toy residual MLP per direction & Same architecture \\
            Network output\tnote{c} & Control & Same parametrization \\
            EMA decay for sampling & \(0.99\) & \(0.99\) \\
            Batch size (paths) & \(1{,}024\) & \(1{,}000\) \\
            Cache size (paths)\tnote{d} & \(\{2^{15}, \bold{2^{18}}\}\) & \(4{,}000\) \\
            Optimizer & AdamW\tnote{e} with \(\beta=(0.9, 0.999)\) & Same optimizer \\
            Initial learning rate & \(\{\bold{3},6,12\}\times10^{-4}\) & \(6\times10^{-4}\) \\
            LR multiplier every 1,000 directional updates & \(\{\boldsymbol{0.8},\sqrt{0.8},1\}\) & \(0.8\) \\
            Alternating stages & Set by training budget & 20 (one direction in final stage) \\
            Updates per direction & 800 per stage & 800 per stage \\
            Divergence calculation & Hutchinson (ablate exact) & Hutchinson \\
            \bottomrule
        \end{tabularx}
        \begin{tablenotes}[flushleft]
            \footnotesize
            \item[a] See \citep[Figure~1]{deng2024reflected}. Values are from the released code and configuration.
            \item[b] Integrated over \([0,1]\), per coordinate.
            \item[c] The learned drift contribution is \(\sigma \times \text{control}\).
            \item[d] See the \hyperref[sec:appendix-cache-size]{cache size paragraph}.
            \item[e] \citet{adamw}
        \end{tablenotes}
    \end{threeparttable}
\end{table}

\paragraph{Model architecture.}
The model passes the sample through a ResNet with three blocks, each block three layers deep with hidden dimensionality 256 and using SiLU activations.
The timestep is embedded using sinusoidal embeddings to 128 dimensions, passed to a two-layer SiLU MLP, and then added to the ResNet output.
The head is a two-layer MLP with a SiLU activation, reducing the dimensionality back to 2.
The network does not take the direction as argument.
Hence, a separate instance is trained in each sampling direction, for both RSBM and RSB-FBSDE.

\paragraph{Cache size.}\label{sec:appendix-cache-size}
In the upstream reference, 4,000 sample paths is generated and stored as cache.
The cash is then used for 800 updates, picking 1,000 full paths at each update.
This amounts to reusing a sample path 200 times on average.
However, since sampling time is a small fraction of the total training time (with cache size \(2^{18}\) only 2\%), we argue RSB-FBSDE stands to gain from less autocorrelated gradient estimates with this relatively cheap training budget tradeoff.
To test this hypothesis, we also included a cache size roughly one order of magnitude smaller in the sweep, which spends 1\% of the training time on sampling.

\paragraph{RSBM sweep.}
Of the parameters included in the RSB-FBSDE sweep, only the learning rate is applied to RSBM.
First, the pretraining is run for 15 minutes for the different learning rates, selecting pretraing learning rate and number of iterations based on the marginal fits only.
With the pretraining configuration fixed, the finetuning learning rate is then swept with the 1h training budget, and a winner selected based on the joint marginal and control cost evaluations as explained above.
When running the sweep, the number of pretrain steps, pretrain learning rate and finetune learning rates were determined to 20,000, \(0.0006\), and \(0.0003\), respectively. 9 terms were used in the reflection summation throughout.

\paragraph{Compute.}
The experiments were run on 40GB A100s, with all three seeds of each experiment configuration running in parallel on one GPU, using CUDA Multi-Process Service to improve GPU utilization.

\paragraph{RSB-FBSDE ablation.}
To see how sensitive the comparison is with respect to some hyperparameters of RSB-FBSDE not included in the sweep, we perform a small ablation in three different directions: divergence calculation, boundary evaluation point for the local time integral in the RSB-FBSDE loss, and time quadrature of the loss.
The last two are explained in more detail in the next section.
The results are presented in Table~\ref{tab:ablation}.
Over all, the resulting evaluation scores barely change between the different configurations.
The only noteworthy difference is that we pay nothing in terms of fit for a gain of more than \(1/3\) in training speedup when using the Hutchinson estimator instead of exact divergence.

\begin{table}[t]
    \caption{RSB-FBSDE ablation. 2-Wasserstein distance and control cost in the checkerboard experiment for various run configurations of RSB-FBSDE. Values are presented on the format mean \(\pm\) one sample standard deviation over three seeds.}
    \label{tab:ablation}
    \centering
    \small
    \begin{tabular}{@{}lccccc@{}}
        \toprule
        & \multicolumn{2}{c}{\(\Pi_0 \to \Pi_1\)} & \multicolumn{2}{c}{\(\Pi_0 \gets \Pi_1\)} & \\
        \cmidrule(lr){2-3} \cmidrule(lr){4-5}
        & \(\calW_2\) \textdownarrow & Cost \textdownarrow & \(\calW_2\) \textdownarrow & Cost \textdownarrow & Time (h) \textdownarrow \\
        \midrule
        Default & 0.034 $\pm$ 0.005 & 2.08 $\pm$ 0.03 & 0.023 $\pm$ 0.005 & 1.33 $\pm$ 0.03 & 1.93 $\pm$ 0.03 \\
        \midrule
        Exact divergence & 0.034 $\pm$ 0.005 & 2.09 $\pm$ 0.02 & 0.023 $\pm$ 0.005 & 1.33 $\pm$ 0.03 & 3.13 $\pm$ 0.06 \\
        Boundary crossing & 0.034 $\pm$ 0.005 & 2.08 $\pm$ 0.03 & 0.023 $\pm$ 0.005 & 1.33 $\pm$ 0.03 & 1.94 $\pm$ 0.02 \\
        Left quadrature & 0.034 $\pm$ 0.006 & 2.09 $\pm$ 0.03 & 0.023 $\pm$ 0.005 & 1.33 $\pm$ 0.04 & 1.92 $\pm$ 0.02 \\
        Right quadrature & 0.034 $\pm$ 0.005 & 2.07 $\pm$ 0.03 & 0.023 $\pm$ 0.005 & 1.31 $\pm$ 0.03 & 1.91 $\pm$ 0.04 \\
        \bottomrule
    \end{tabular}
\end{table}

\subsection{RSB-FBSDE implementation}
In this section, we briefly motivate some implementation details of RSB-FBSDE \citep{deng2024reflected}.

In all our RSB-FBSDE experiments, the loss for the backward sampling control \(\overleftarrow{v}_{\omega}\) is taken as
\begin{align*}
    \calL (\omega) \coloneq & \int_0^1 \bE\left[\frac{1}{2} \| \overleftarrow{v}_{\omega}(t, X_t) \|_2^2 + \sigma(t) \nabla \cdot \overleftarrow{v}_{\omega}(t, X_t) + \overrightarrow{v}_{\theta}(t, X_t) \cdot \overleftarrow{v}_{\omega}(t, X_t) \right] dt \\
    &+ 2 \, \bE \left[ \int_0^1 \frac{1}{\sigma(t)} \overleftarrow{v}_{\omega}(t, X_t) \cdot n(X_t) dL_t \right]
\end{align*}
where expectation is with respect to the entire forward path of 
\[
    dX_t = (\mu(t, X_t) + \sigma (t) \overrightarrow{v}_{\theta}(t, X_t)) dt + \sigma(t) dW_t + n(X_t) dL_t, \quad X_0 \sim \pi_0,
\]
sampled using the forward control \(\overrightarrow{v}_{\theta}\) (kept fixed).
The corresponding loss for the forward control follows equivalently.
The factor \(2\) of the local time integral, differing from Algorithm~1 of \citet{deng2024reflected}, is motivated by the following:
Assume \(X\) is the SB and \(\overrightarrow{v}_{\theta^*}\) has been found, so that only finding \(\overleftarrow{v}_{\omega^*}\) remains.
Writing the loss instead as \(\calL (v)\), we have \(\frac{d}{d \varepsilon} \calL (\overleftarrow{v}_{\omega^*} + \varepsilon h) = 0\) at \(\varepsilon = 0\) for all smooth perturbations \(h\) if and only if the factor is 2.
Using partial integration on the divergence term leads to a boundary integral which may be turned into an expectation of a local time integral \citep[see e.g.][eq.~(45)]{gobet2001euler}.
The local time term of \(\calL \) precisely cancels this boundary contribution.

When discretizing \(\calL\), the cross term \(\overrightarrow{v}_{\theta} \cdot \overleftarrow{v}_{\omega}\) poses a subtle challenge.
Due to the integration being in time, it might be natural to employ a left sum.
The problem is that \(\overleftarrow{v}_{\omega}\) is then never trained at \(t=1\). In the subsequent step of the algorithm, the backwards sampling then must rely on an extrapolation.
On the flip side, the right sum solves this problem but then introduces another one for the subsequent step:
Now, \(\overleftarrow{v}_{\omega}\) is never trained at \(t=0\).
This is fine for sampling the backwards process since we use Euler-Maruyama to integrate the SDE.
But when training \(\overrightarrow{v}_{\theta}\) at \(t=0\) (time has now switched direction so this is the last term in the right sum), the loss depends on \(\overleftarrow{v}_{\omega}(0, \cdot)\), so we are again dependent on extrapolation.
Therefore, we choose trapezoidal quadrature, and compare with right and left quadrature in the ablation study.

For the local time integral, we follow \citet[Prop.~4.3]{nordenhog2025} by using the symmetrized reflected method and evaluating the integrand in coordinate \(x\) at the boundary reflection point, defined to be the projection of an out-of-bounds proposal onto the domain \citep[eq.~(A.13)]{nordenhog2025}.
In the \(t\)-coordinate, to be consistent with the first term in the loss \(\calL\), we generalize their form slightly by using the same quadrature in both.
Their method is recovered when using the left sum.

To align with the theory of \citet{nordenhog2025}, we define the boundary evaluation point as the projected point of an Euler-Maruyama proposal that falls outside the domain.
The method employed by \citet{deng2024reflected} is not detailed, but their Algorithm~4 suggests they define the boundary evaluation point as the point where the line segment of the previous state and the out-of-bounds proposal intersects the boundary (in Table~\ref{tab:ablation} called \textit{boundary crossing}).
For a rectangular domain and normal reflection, one could imagine this as the point where the billiard ball hits the edge, whereas the projection is the mid point between were it ended up and where it would have ended up if not for the wall.
We compare these in the ablation study.

\subsection{Image-to-image}

We mimic the experimental setup of \citet{debortoli2024schrodinger}.
As architecture we start from the U-Net implementation of \citet[MIT License]{dhariwal2021diffusion} in PyTorch \citep{paszke2019pytorch}, extended to include conditioning on a directional embedding.
Adam \citep{kingma2014adam} with \(\beta=(0.9, 0.999)\) is used as optimizer and gradients are clipped if exceeding unit norm.

In both experiments, we sample using an EMA model with decay rate \(0.999\), both in training in the finetuning step and for evaluation. We used 11 reflections per dimension in the truncation of \eqref{eq:reflected-transition-density}.

\paragraph{Datasets}
MNIST contains 60,000 training samples and 10,000 test samples.
For EMNIST we used the split \texttt{letters}, yielding 24,000 training and 4,000 test samples.
In AFHQ, the \texttt{cat} and \texttt{wild} classes contained, respectively, 5065 and 4593 training images, 493 and 483 test images.

\paragraph{Model architecture.}
For the model, we use the U-Net implementation of \citet[MIT License]{dhariwal2021diffusion}, which we augment by concatenating the directional embedding to the time embedding.
The directional embedding is defined in the same way as the time embedding, namely a two-layer MLP on top of sinusoidal embeddings.
These are then concatenated and used to FiLM-modulate \citep{perez2018FiLM} GroupNorm layers in the residual blocks.
The parameters to the models and those used in training are presented in Table~\ref{tab:hyperparameters-img}.

\begin{table}[h]
    \centering
    \begin{threeparttable}
        \caption{Image experiment hyperparameters}
        \label{tab:hyperparameters-img}
        \begin{tabular}{lcc}
            \toprule
             & MNIST \(\leftrightarrow\) EMNIST & AFHQ \\
             \cmidrule(r){2-2} \cmidrule(l){3-3}
             Sample dimension & \(1 \times 28 \times 28\) & \(3 \times 64 \times 64\) \\
             \(\sigma\) & 1.00 & 0.75 \\
            Euler-Maruyama discretization steps & 30 & 100 \\
            Model channels & 64 & 128 \\
            Channel multipliers & 1, 2, 2 & 1, 2, 3, 4 \\
            Number of residual blocks & 2 & 4 \\
            Attention resolutions & -- & 32, 16, 8 \\
            Number of head channels & -- & 64 \\
            Dropout & 0.1 & 0.0 \\
            Batch size & 128 & 128 \\
            Learning rate & \(1 \times 10^{-4}\) & \(2 \times 10^{-4}\) \\
            Gradient clipping (max.\ norm)\tnote{a} & 1 & \(16/9\) \\
            Warm-up period (linear schedule) & -- & 5,000 \\
            \(N_\text{pretrain}\) & 100,000 & 100,000 \\
            \(N_\text{finetune}\) & 150,000 & 20,000 \\
             \bottomrule
        \end{tabular}

        \begin{tablenotes}
            \small \item[a] Equals \(\sigma^{-2}\) to align with the unit norm used in \citet{debortoli2024schrodinger}, since the losses are scaled differently.
        \end{tablenotes}
    \end{threeparttable}
\end{table}

\paragraph{Compute.}
The experiments were run on a GH200 96GB.
We wish to highlight that training \(\alpha\)-RSBM is only marginally slower than its non-reflected counterpart, with MNIST-to-EMNIST taking around 22.5~h vs. 21.75~h and AFHQ 52.5~h vs. 52~h.

\paragraph{DSBM out-of-bounds statistics.} We conclude this section by providing some statistics for pixel or channel value violations, displayed in Table~\ref{tab:image-oob}, to get a sense of the scale of the problem of sampling out-of-bounds with \(\alpha\)-DSBM.
Out of the generated MNIST and EMNIST samples, all had at least one pixel value outside the domain \([0, 1]\).
In fact, \(50 \%\) of the MNIST pixels where in violation of the domain constraint.
The average violation among these, however, was relatively small (less than \(0.3 \%\)).
For AFHQ, much fewer channel values (around \(2 \%\)) where out of bounds, but with a higher average violation.
And despite the overall rate being low, \(70\%\) -- \(80\%\) of generated images had at least one channel of one pixel in violation.

\begin{table}[t]
    \caption{Out-of-bounds (OOB) statistics for \(\alpha\)-DSBM. For each sample in the test datasets, we generate one sample from the opposite distribution. This is repeated for three seeds, retraining from scratch for each. For AFHQ, a channel value is an R, G or B value for a pixel; for (E)MNIST it equates to pixel value. Mean overshoot is the average distance to \([0, 1]\) among OOB channel values.}
    \label{tab:image-oob}
    \centering
    \small
    \setlength{\tabcolsep}{4pt}
    \begin{tabular}{@{}lcccc@{}}
        \toprule
        Target & Generated images & OOB channel values (\%) & Images with OOB (\%) & Mean overshoot \\
        \midrule
        MNIST & 12,000 & 50.30 & 100.00 & 0.00272 \\
        EMNIST & 30,000 & 34.86 & 100.00 & 0.00233 \\
        AFHQ \texttt{cat} & 1,449 & 2.11 & 78.74 & 0.02178 \\
        AFHQ \texttt{wild} & 1,479 & 1.52 & 70.18 & 0.02159 \\
        \bottomrule
    \end{tabular}
\end{table}

\newpage
\section{Additional results}
\label{sec:appendix-additional-results}

\begin{figure}[h]
    \centering
    \begin{subfigure}{\linewidth}
        \centering
        \includegraphics{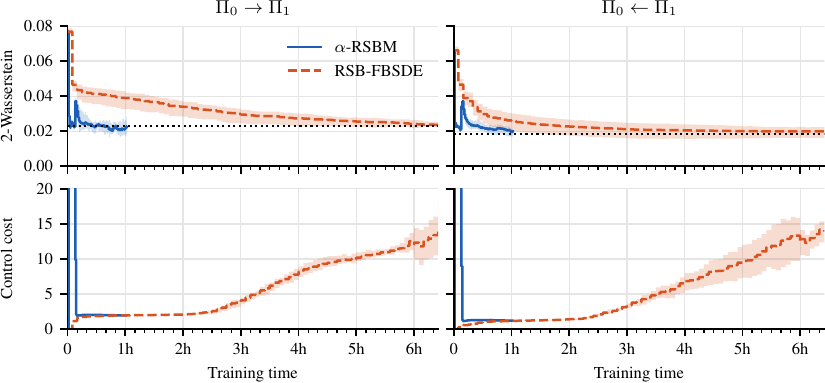}
        \caption{Uniform--checkerboard. All RSB-FBSDE seeds stopped 55--59\% of the way through, and the control cost of $\alpha$-RSBM during pretraining lies above the plotted range.}
        \label{fig:long-run-checkerboard}
    \end{subfigure}

    \medskip
    \begin{subfigure}{\linewidth}
        \centering
        \includegraphics{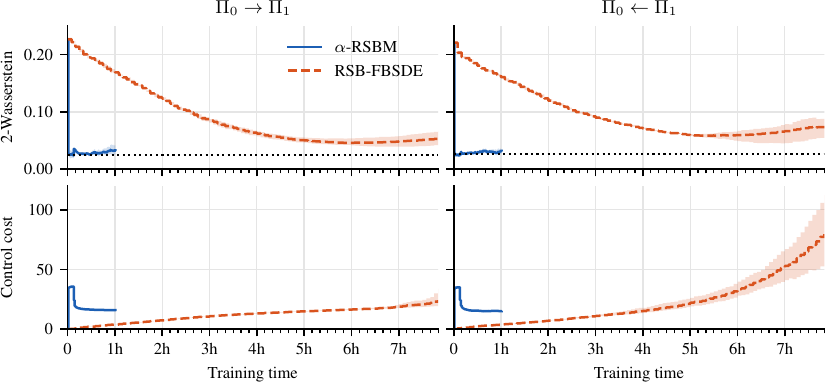}
        \caption{Schrödinger--Öresund. All RSB-FBSDE seeds stopped 68--76\% of the way through.}
        \label{fig:long-run-oresund}
    \end{subfigure}
    \caption{Comparison of $\alpha$-RSBM and a longer run of RSB-FBSDE. Lines show the mean and shaded bands the minimum and maximum over three seeds. Dotted lines mark the threshold, the 95th percentile of the mean \(\calW_2\) distance over all triples of 300 batches, each batch containing a \(\calW_2\) estimate from 4,096 independent true samples. Compared with the two-hour RSB-FBSDE runs, these have 120 rather than 20 stages and a learning rate that decays six times more slowly. All their seeds stopped early, when the drift exploded, and their curves end at the last evaluation of the first seed to stop. $\alpha$-RSBM trained for its one-hour budget.}
    \label{fig:long-run}
\end{figure}

\begin{figure}
    \centering
    \begin{subfigure}{0.48\textwidth}
        \includegraphics[width=1\linewidth]{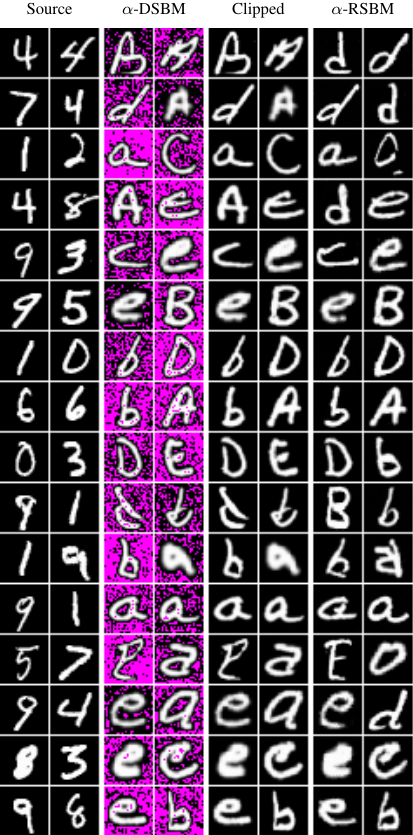}
        \caption{\texttt{MNIST} \(\to\) \texttt{EMNIST}}
    \end{subfigure}
    \vspace{2mm}
    \hfill
    \begin{subfigure}{0.48\textwidth}
        \includegraphics[width=1\linewidth]{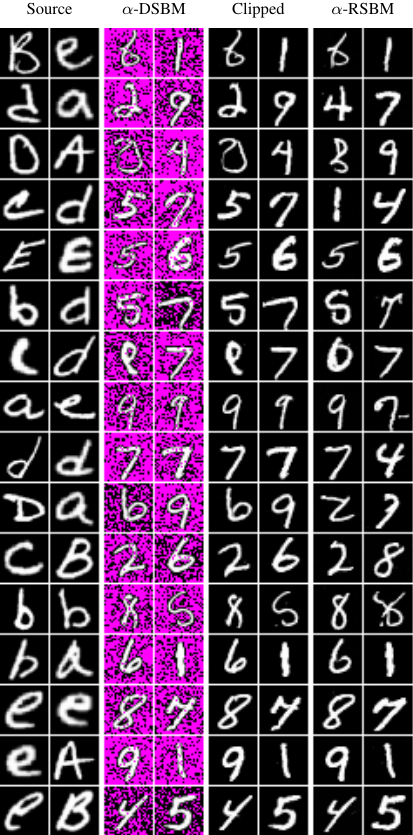}
        \caption{\texttt{EMNIST} \(\to\) \texttt{MNIST}}
    \end{subfigure}
    \begin{subfigure}{0.48\textwidth}
        \includegraphics[width=1\linewidth]{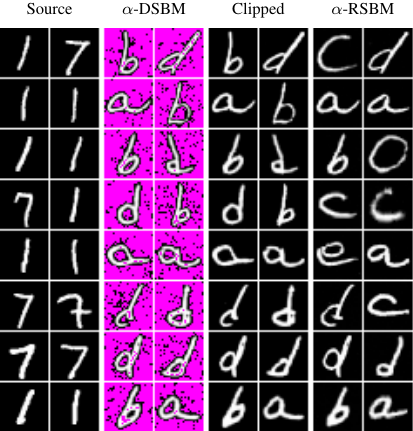}
        \caption{\texttt{MNIST} \(\to\) \texttt{EMNIST}}
    \end{subfigure}
    \hfill
    \begin{subfigure}{0.48\textwidth}
        \includegraphics[width=1\linewidth]{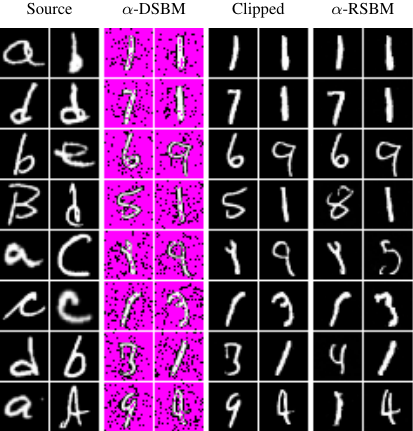}
        \caption{\texttt{EMNIST} \(\to\) \texttt{MNIST}}
    \end{subfigure}
    \caption{
        MNIST-to-EMNIST transfer, where (c) and (d) contains the 16 samples for which \(\alpha\)-DSBM produced highest out-of-bounds pixel count.
        Out-of-bounds pixels are colored in magenta.
    }
    \label{fig:mnist-more-random-samples-incl-top-oob}
\end{figure}

\begin{figure}
    \centering
    \begin{subfigure}{0.48\textwidth}
        \includegraphics[width=1\linewidth]{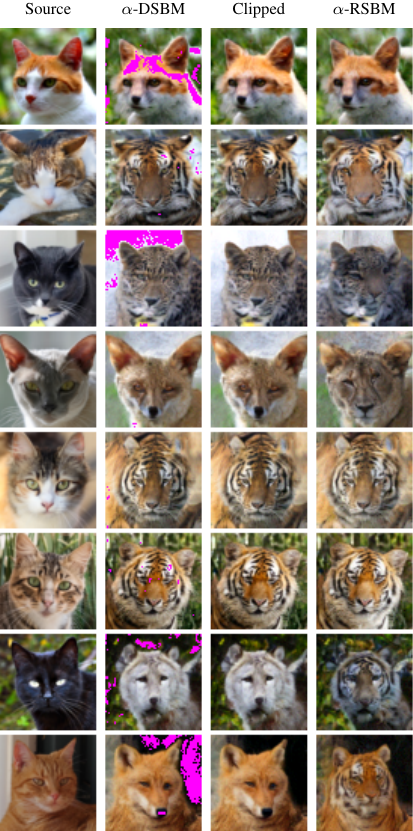}
        \caption{\texttt{cat} \(\to\) \texttt{wild}}
    \end{subfigure}
    \vspace{2mm}
    \hfill
    \begin{subfigure}{0.48\textwidth}
        \includegraphics[width=1\linewidth]{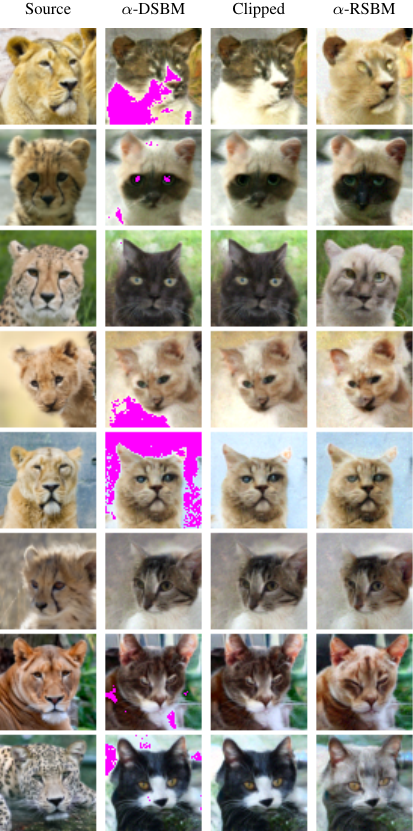}
        \caption{\texttt{wild} \(\to\) \texttt{cat}}
    \end{subfigure}
    \begin{subfigure}{0.48\textwidth}
        \includegraphics[width=1\linewidth]{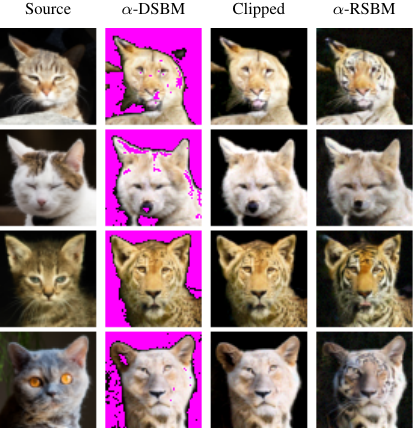}
        \caption{\texttt{cat} \(\to\) \texttt{wild}}
    \end{subfigure}
    \hfill
    \begin{subfigure}{0.48\textwidth}
        \includegraphics[width=1\linewidth]{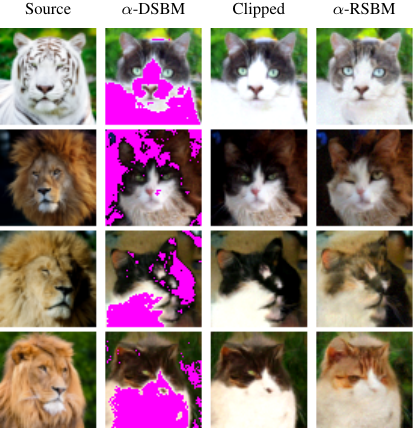}
        \caption{\texttt{wild} \(\to\) \texttt{cat}}
    \end{subfigure}
    \caption{
        AFHQ (\(64 \times 64\)) transfer, where (c) and (d) contains the four samples for which \(\alpha\)-DSBM produced highest out-of-bounds pixel count.
        Out-of-bounds pixels are colored in magenta.
    }
    \label{fig:afhq64-more-random-samples-incl-top-oob}
\end{figure}


\end{document}